\documentclass{article}

\usepackage{times}
\usepackage{graphicx} 
\usepackage{subfigure}

\usepackage{natbib}

\usepackage{algorithm}
\usepackage{algorithmic}
\usepackage[algo2e,lined,boxed,vlined]{algorithm2e}

\usepackage{amsfonts}
\usepackage{amssymb}
\usepackage{booktabs}

\usepackage{hyperref}
\usepackage{xurl}

\usepackage[accepted]{icml2016}

\usepackage{comment}
\usepackage{amsthm}
\usepackage{mathtools}
\usepackage{tabularx}
\usepackage{multirow}

\newtheorem{lemma}{Lemma}
\newtheorem{proposition}{Proposition}

\newtheorem{remark}{Remark}

\def\b{\ensuremath\boldsymbol}

\usepackage{lastpage}
\usepackage{fancyhdr}
\usepackage{url}
\icmltitlerunning{Data Evolution by Wittgenstein's Rule Following}

\begin{document}


\twocolumn[
\icmltitle{Data Evolution by Wittgenstein's Rule Following}

\icmlauthor{Aydin Ghojogh}{aghojogh@lakeheadu.ca}

\icmlauthor{Benyamin Ghojogh}{bghojogh@uwaterloo.ca}


\vskip 0.3in
]

\begin{abstract}
This paper introduces Wittgenstein's Rule Following (WRF) data evolution, a framework in philomatics for evolving or generating a new dataset from a sequence of previously observed datasets. The method is inspired by Ludwig Wittgenstein's rule-following considerations and his notion of family resemblance in \textit{Philosophical Investigations}. Unlike standard synthetic data generation, where the goal is usually to sample from or augment a fixed distribution, WRF aims to continue the implicit rule expressed by a historical sequence of datasets while preserving resemblance to the previous datasets.

WRF represents each dataset by structural descriptors rather than pointwise correspondences. These descriptors summarize geometric, distributional, clustering, and, in the supervised case, label-based properties of the data. The method predicts a rule-following target by extrapolating descriptor trajectories and a family-resemblance target by averaging historical descriptors. Candidate datasets are then generated from the observed history through balanced or bounded mixture recombination, scored according to these targets, and optionally refined through differentiable optimization in descriptor space.

The proposed framework allows both sample size and feature dimension to vary over time and does not assume that the next dataset is a direct transformation of the last one. Simulations on synthetic and image datasets show that WRF can generate meaningful continuations of evolving datasets in both unsupervised and supervised settings.
\end{abstract}

{\textbf{\textit{Keywords---}} data evolution, data generation, sequence of datasets, rule following, family resemblance, Ludwig Wittgenstein, philomatics, machine learning.}

\hfill\break
{\textbf{\textit{Code---}} \url{https://github.com/aghojogh/WRF_data_evolution}}

\section{Introduction}

Many machine learning problems involve not only analyzing a fixed dataset, but also understanding how data may evolve over time. In such settings, we may observe a sequence of datasets and wish to generate a plausible next dataset which continues the structural pattern expressed by the previous ones. This problem differs from ordinary sample generation because the goal is not merely to sample more data from one fixed distribution. 
This distinguishes the present setting from standard synthetic data generation and data augmentation, where the aim is usually to enlarge or transform an existing dataset rather than to continue a sequence of changing datasets \cite{goodfellow2014generative, kingma2013auto, chen2023unified, iwana2021empirical}.

Rather, the goal here is to continue a rule implicit in a historical sequence of datasets. The datasets in this sequence may have different numbers of samples, different feature dimensions, and, in the supervised case, different label structures. 
Therefore, the next dataset cannot always be modeled as a direct pointwise transformation of the last dataset.
This setting is also related to learning under changing or nonstationary data distributions, often studied under concept drift and data-stream learning \cite{gama2014survey,lu2018learning}.

Recently, some machine learning algorithms have been developed that make use of philosophical concepts \cite{ghojogh2022affective}.
Specifically, \textit{philomatics} has been introduced as a framework for integrating philosophy and mathematics, including machine learning as a mathematical discipline \cite{ghojogh2023philomatics}. 
In this paper, we contribute to the growing line of research that develops artificial intelligence through ideas inspired by Wittgenstein’s philosophy \cite{miller2021does,ghojogh2023philomatics,molino2023witgenstein,amanpour2026wittgenstein}.

This paper proposes a philomatically motivated algorithm \cite{ghojogh2023philomatics} for addressing the abovementioned problem, drawing inspiration from Ludwig Wittgenstein’s rule-following considerations and his notion of family resemblance. In Wittgenstein’s philosophy, following a rule is not simply a matter of applying a private formula that uniquely determines all future cases. Instead, rule-following is grounded in practice, examples, correction, and the ability to continue in a way that is accepted as correct within a shared practice. At the same time, Wittgenstein’s notion of family resemblance suggests that members of a family need not share one common essence; rather, they may be connected through overlapping similarities. These two ideas motivate the central viewpoint of this paper: the next dataset should continue the rule expressed by the historical sequence while also preserving resemblance to the family of previously observed datasets.

We formulate this idea as the problem of rule-following data evolution. Given a sequence of observed datasets:
\begin{align*}
\mathcal{X}_{1:T} := \{\b{X}_1,\ldots,\b{X}_T\},
\end{align*}
and, in the supervised case, a corresponding sequence of label vectors:
\begin{align*}
\mathcal{Y}_{1:T} := \{\b{y}_1,\ldots,\b{y}_T\},
\end{align*}
the goal is to generate an evolved dataset $\widehat{X}_{T+1}$, and optionally its labels $(\widehat{y}_{T+1})$. The method is required to use only the observed historical datasets and their labels, when available. It does not assume that the next dataset is obtained from the last dataset by a known transformation, nor does it require pointwise correspondence among datasets.

To address this problem, we propose Wittgenstein’s Rule Following (WRF) data evolution. The main idea is to represent each dataset by a finite-dimensional descriptor vector that summarizes its structural properties. In the unsupervised case, the descriptor contains information about sample size, feature dimension, pairwise-distance distribution, covariance structure, principal variance spectrum, and clustering tendency. In the supervised case, an additional descriptor summarizes the label organization, including the number of classes, class entropy, class imbalance, within-class compactness, between-class separation, and class separability. These descriptors allow datasets of different sizes and, after suitable alignment, different dimensions to be compared at a structural level rather than by point-wise correspondence.

WRF combines two complementary targets in descriptor space. The first is a rule-following target, obtained by extrapolating the trajectory of historical descriptors. This target encourages the evolved dataset to continue the trend expressed by the sequence. The second is a family-resemblance target, obtained by averaging the historical descriptors. This target prevents the generated dataset from drifting too far from the family of observed datasets. Thus, the algorithm balances continuation and resemblance: it tries to go forward according to the observed rule while remaining connected to the historical family.

The generation procedure is history-based. Candidate datasets are constructed only from the observed datasets, after mapping them to a predicted target dimension. We consider balanced family generation, where each historical dataset contributes approximately equally, and bounded mixture generation, where random mixture proportions are used but every previous dataset receives a non-negligible contribution. Small noise may be added so that candidates are not merely exact resamplings of past data. Each candidate is then scored by a weighted loss measuring its agreement with the rule-following target, the family-resemblance target, the last observed dataset, the predicted shape, and a non-collapse criterion. In the supervised case, additional label-aware losses are included.

After selecting the best candidate, WRF can optionally refine it by differentiable optimization. Since some descriptor components, such as quantiles, clustering assignments, and silhouette scores, are not convenient for gradient-based optimization, the refinement step uses differentiable surrogate descriptors. The coordinates of the selected candidate are optimized so that the evolved dataset better matches the rule-following and family-resemblance targets, while the labels, if present, remain fixed. This produces a final evolved dataset that is both historically grounded and structurally adjusted.

The main contributions of this paper are as follows. First, we introduce the problem of rule-following data evolution, where the aim is to generate a continuation of a sequence of datasets rather than to model a single fixed dataset. Second, we propose WRF, a philomatically motivated algorithm inspired by Wittgenstein’s rule following and family resemblance. Third, we develop unsupervised and supervised dataset descriptors that allow comparison of datasets without pointwise correspondence. Fourth, we introduce a history-based candidate generation, scoring, and refinement framework for evolving datasets with possibly varying sample sizes and feature dimensions. Finally, we analyze the algorithmic complexity of WRF and demonstrate its behavior through simulations in both unsupervised and supervised settings.

The remainder of this paper is organized as follows. Section \ref{section_background} reviews the philosophical background on Wittgenstein’s rule following. Section \ref{section_problem_definition} formalizes the problem of rule-following data evolution. Section \ref{section_dataset_descriptors} defines the unsupervised and supervised dataset descriptors. Section \ref{section_rule_prediction} introduces rule-following and family-resemblance targets in descriptor space for rule prediction. Section \ref{section_candidate_generation} describes history-based candidate generation. Section \ref{section_candidate_scoring} presents the candidate scoring procedure. Section \ref{section_data_refinement} explains the differentiable data refinement step. Section \ref{section_algorithm_complexity} summarizes the full WRF algorithm and analyzes its time and space complexities. Section \ref{section_simulations} presents simulations, and Section \ref{section_conclusion} concludes the paper.

\section{Background on Wittgenstein's Rule Following in Analytic Philosophy}\label{section_background}

This section introduces Ludwig Wittgenstein and his philosophy, which has inspired the development of this algorithm.
We also explain rule following discussed by Wittgenstein. A numerical example is also provided for better understanding of Wittgenstein's rule following.

\subsection{Introduction to Ludwig Wittgenstein's Philosophy}

\subsubsection{Biography of Ludwig Wittgenstein}

Ludwig Wittgenstein (1889--1951) was an Austrian-British philosopher whose work had a major impact on logic, the philosophy of language, and the philosophy of mind. He first studied engineering in Berlin before going to Cambridge, where he worked with Bertrand Russell. He later held a professorship in philosophy at the University of Cambridge. Wittgenstein became one of the most influential figures in twentieth-century analytic philosophy \cite{wittgenstein1990ludwig}.

\subsubsection{Early Wittgenstein (1911--1929)}

Wittgenstein's early period, roughly from 1911 to 1929, is represented most famously by his \textit{Tractatus Logico-Philosophicus}, first published in 1921 \cite{wittgenstein1921tractatus}. In the \textit{Tractatus}, Wittgenstein presents a picture theory of language, according to which meaningful propositions picture facts by sharing a logical structure with reality. The book attempts to determine the limits of meaningful language and to dissolve philosophical problems by showing that many of them arise from attempts to say what cannot meaningfully be said. The work is highly systematic and is organized through a hierarchical numbering of propositions, a structure often compared to Spinoza's \textit{Ethics} \cite{spinoza1677ethics}. Its central concern is the relation between language, logical form, and the world \cite{stanford2002wittgenstein}.

\subsubsection{Later Wittgenstein (1930--1951)}

After withdrawing from academic philosophy for several years, Wittgenstein returned to Cambridge and gradually moved away from many of the central ideas of the \textit{Tractatus}. His later philosophy, developed mainly from 1930 until his death in 1951, is most clearly expressed in \textit{Philosophical Investigations}, published posthumously in 1953 \cite{wittgenstein1953philosophical}. In this later work, Wittgenstein rejects the idea that language has one single underlying logical essence. Instead, he emphasizes that meaning depends on the use of words within particular contexts and practices. This view is developed through concepts such as ``language-games'' and ``forms of life.'' For the later Wittgenstein, philosophy is not primarily the construction of a theory, but a therapeutic activity: it clarifies how language is actually used in order to remove conceptual confusion \cite{stanford2002wittgenstein}.

\subsection{Wittgenstein's Rule Following}

Wittgenstein's rule-following considerations are developed mainly in
\emph{Philosophical Investigations} \cite{wittgenstein1953philosophical}, \S\S 185--242\footnote{It is a standard scholarly shorthand, especially in philosophy and law. It means from section 185 through section 242 (inclusive). Note that, in philosophical manuscripts, \S\, and \S\S\, refer to Section and Sections, respectively.}.
In these passages, Wittgenstein investigates what it means for a person to follow a rule correctly. His discussion begins from cases such as continuing a numerical series, for example continuing by the rule ``+2'' \citep[\S\S185--190]{wittgenstein1953philosophical}. At first, one may think that the rule itself, or the agent's inner understanding of it, already determines all future applications. Wittgenstein challenges this view by showing that no interpretation of a rule can, by itself, fix its application in every possible case. Any interpretation can itself be interpreted in different ways, and therefore interpretation alone cannot serve as the final ground of correct rule-following \citep[§§198--201]{wittgenstein1953philosophical}.

The central formulation of the rule-following problem appears in §201, where Wittgenstein argues that, if rule-following were merely a matter of interpretation, then any course of action could be made to agree with the rule under some interpretation \citep[§201]{wittgenstein1953philosophical}\footnote{This issue became central in later debates on the rule-following paradox, especially through Kripke's interpretation of Wittgenstein \cite{kripke1982wittgenstein}. Kripke argues that no finite record of a person's past applications of a rule is sufficient, by itself, to determine a unique future application of that rule. For example, even if a person has always used addition correctly for small numbers, one may still construct an alternative rule that agrees with all previous uses but differs in a new case. This skeptical challenge emphasizes the gap between past behavior and future correctness. In Kripke's reading, Wittgenstein's response is not to locate the meaning of a rule in a private mental fact, but to understand rule-following through public practice, communal agreement, and shared standards of correction.}. 
This is not meant to show that rules are useless or that correctness is impossible. Rather, Wittgenstein's point is that the application of a rule is not grounded in a private mental act standing behind the action. A rule has force only within a practice: it is learned, applied, corrected, and stabilized through use. Thus, to follow a rule is not merely to consult an inner formula, but to participate in a practice in which certain applications are accepted as correct and others as incorrect \citep[§§199--202]{wittgenstein1953philosophical}.

This is why Wittgenstein connects rule-following with custom, training, agreement, and forms of life. The correctness of following a rule is not established by a purely private standard, since a private standard would not provide a genuine distinction between seeming correct and being correct. Instead, rule-following depends on public criteria of use\footnote{Different interpretations of this point have been discussed extensively in the secondary literature on Wittgenstein's rule-following considerations \cite{kripke1982wittgenstein,mcdowell1984wittgenstein,baker1984scepticism}.}. 
One learns how to go on in the same way through examples, correction, habit, and shared human practices \citep[§§202, 217--219]{wittgenstein1953philosophical}. In this sense, Wittgenstein's rule-following considerations are closely connected to his broader view that meaning is use. The meaning of a rule is not an abstract object detached from practice, but is expressed in the way the rule is actually employed within a language-game \citep[§§23, 43, 199]{wittgenstein1953philosophical}.

The discussion culminates in Wittgenstein's emphasis on agreement in judgments and forms of life. Human beings can follow rules because their actions, responses, and judgments are embedded in shared forms of activity. Agreement is not merely agreement in definitions, but agreement in the ways people act, judge, correct, and continue practices \citep[§§240--242]{wittgenstein1953philosophical}. Therefore, rule-following reveals a deep connection between meaning, action, and social practice. A rule does not determine its applications through a hidden private mechanism; rather, its normativity arises from its role in a public practice governed by training, use, and communal standards of correctness.

\subsection{An Example for Wittgenstein's Rule Following}


As mentioned above, Wittgenstein's discussion of rule following begins with examples involving the continuation of a numerical series \citep[§§185--190]{wittgenstein1953philosophical}. To clarify the idea, consider the following finite sequence:
\begin{align*}
2, 4, 8, \dots
\end{align*}
What should the next number be? This is the kind of problem that primary-school students are sometimes asked to solve by their teachers.

One natural answer is $16$. A person may infer that the sequence follows a geometric rule with common ratio $2$, since $4=2\times 2$ and $8=4\times 2$. Under this interpretation, the next term is:
\begin{align*}
8\times 2 = 16.
\end{align*}

However, another person may infer a different rule. For example, one may say that the sequence is generated by alternately adding $2$ and $4$:
\begin{align*}
2,\quad 2+2=4,\quad 4+4=8,\quad 8+2=10,\quad \ldots
\end{align*}
Under this interpretation, the next number is $10$, not $16$.


Many more complex rules can also be constructed. For example, consider the cubic polynomial function:
\begin{align*}
f(n) = 1165 n^3 - 6989 n^2 + 12814 n - 6988,
\end{align*}
where $n$ denotes the index of the term, starting from $n=1$.
Then, we can say that:
\begin{itemize}
\item $n=1 \implies f(1) = 1165 (1) - 6989 (1) + 12814 (1) - 6988 = 2$
\item $n=2 \implies f(2) = 1165 (8) - 6989 (4) + 12814 (2) - 6988 = 4$
\item $n=3 \implies f(3) = 1165 (27) - 6989 (9) + 12814 (3) - 6988 = 8$
\item $n=4 \implies f(4) = 1165 (64) - 6989 (16) + 12814 (4) - 6988 = 7004$
\end{itemize}
Therefore, the next number should be $7004$ according to the rule!

This example illustrates an important point in Wittgenstein's rule-following considerations: a finite sequence does not, by itself, uniquely determine how it must be continued. Many different rules can agree with the same observed initial terms while producing different future continuations. For Wittgenstein, the application of a rule is therefore not fixed merely by an abstract formula or by a private interpretation. Rather, what counts as continuing the rule correctly is stabilized by practice, training, examples, correction, and shared criteria of use.

This observation motivates the data-evolution problem studied in this paper. Given only a finite history of datasets, there may be many possible ways to generate a next dataset that is compatible with the observed history. The proposed WRF method addresses this ambiguity by combining two principles: rule following, which encourages continuation of the structural trend in the historical sequence, and family resemblance, which keeps the generated dataset close to the family of previously observed datasets.

\section{Problem Definition: Data Evolution}\label{section_problem_definition}

Consider an observed sequence of datasets:
\begin{align}
\mathcal{X}_{1:T} := \{\b{X}_1,\ldots,\b{X}_T\} = \{\b{X}_t\}_{t=1}^T,
\end{align}
where the $t$-th dataset is:
\begin{align}
\b{X}_t \in \mathbb{R}^{n_t \times d_t}.
\end{align}
The number of samples $n_t$ and the dimension $d_t$ are allowed to vary with $t$. 

In the supervised case, each dataset is equipped with labels:
\begin{align}
\mathcal{Y}_{1:T} := \{\b{y}_1,\ldots,\b{y}_T\} = \{\b{y}_t\}_{t=1}^T,
\end{align}
where $\b{y}_t \in
\{1, 2, \dots, K_t\}^{n_t}$ is the vector of labels and $K_t$ is the number of classes in the dataset $\b{X}_t$.


The goal is to evolve, or generate, a new dataset:
\begin{align}
\widehat{\b{X}}_{T+1} \in \mathbb{R}^{\widehat{n}_{T+1}\times \widehat{d}_{T+1}},
\end{align}
which continues the rule implicit in the sequence $\{\b{X}_1,\ldots,\b{X}_T\}$. 
In the supervised case, the labels of the data instances in the evolved dataset should also be generated:
\begin{align}
\widehat{\b{y}}_{T+1}
\in
\{1, 2, \dots, K_{T+1}\}^{\widehat{n}_{T+1}},
\end{align}
where $K_{T+1}$ is the number of classes in the dataset $\widehat{\b{X}}_{T+1}$.


Importantly, the algorithm has only access to the sequence of datasets and possibly to their labels in the supervised case. 
Candidate datasets are constructed from the observed history $\{\b{X}_1,\ldots,\b{X}_T\}$, and, in the supervised case, from the label history $\{\b{y}_1,\ldots,\b{y}_T\}$, with optional random padding and small noise injection.

The WRF algorithm is a rule-following data evolution in which the next dataset is not required to be a transformation of the last dataset. Instead, it is selected and refined to satisfy two competing principles:
\begin{itemize}
\item Wittgenstein's rule following: continue the rule of the sequence of datasets
\item Wittgenstein's family resemblance: preserve family resemblance to the observed history
\end{itemize}
The construction is history-based, i.e., the algorithm receives only the previous datasets, and optionally their labels, and generates a continuation using descriptors, balanced historical recombination, candidate scoring, and differentiable refinement (explained in the following).

\section{Dataset Descriptors}\label{section_dataset_descriptors}

The purpose of the descriptor maps is to represent each dataset by a finite-dimensional vector of statistics. This allows datasets with different numbers of samples and, in some cases, different dimensions to be compared at a structural level rather than by pointwise correspondence. Thus, instead of requiring that the rows of \(\b{X}_t\) and \(\b{X}_{t+1}\) correspond to one another, we compare their global geometric, distributional, and, when available, supervised class structures.

\subsection{Standardization of Datasets}

Let $\b{x}_{t,i} \in \mathbb{R}^{d_t}$ denote the $i$-th data instance of dataset $\b{X}_t$:
\begin{align}
\b{X}_t
=
\begin{bmatrix}
\b{x}_{t,1}^\top \\
\vdots \\
\b{x}_{t,n_t}^\top
\end{bmatrix}
\in \mathbb{R}^{n_t \times d_t}.
\end{align}
We first standardize the dataset feature-wise. 
Feature-wise standardization is a standard preprocessing step in machine learning and multivariate data analysis \cite{bishop2006pattern,hastie2009elements}.
Let the mean of the dataset $\b{X}_t$ be:
\begin{align}
\b{\mu}_t
:=
\frac{1}{n_t}
\sum_{i=1}^{n_t} \b{x}_{t,i}
\in \mathbb{R}^{d_t},
\end{align}
and let the vector of coordinate-wise empirical standard deviations be:
\begin{align}
\b{\sigma}_t
:=
[
\sigma_{t,1},\ldots,\sigma_{t,d_t}
]^\top
\in \mathbb{R}^{d_t}.
\end{align}
The standardized instances of $\b{X}_t$ are:
\begin{align}
\b{z}_{t,i}
:=
\frac{\b{x}_{t,i}-\b{\mu}_t}{\b{\sigma}_t+\varepsilon}
\in \mathbb{R}^{d_t},
\end{align}
for all $i \in \{1, \dots, n_t\}$, where the division is performed componentwise and \(\varepsilon>0\) is a small constant for numerical stability. We write the matrix of standardized data instances as:
\begin{align}
\b{Z}_t
:=
\begin{bmatrix}
\b{z}_{t,1}^\top \\
\vdots \\
\b{z}_{t,n_t}^\top
\end{bmatrix}
\in \mathbb{R}^{n_t \times d_t}.
\end{align}
The standardized data instances are used in defining some of the dataset descriptors, which are explained in the following. 

\subsection{Unsupervised Data Descriptor}

Each dataset is mapped to a $p$-dimensional unsupervised descriptor vector:
\begin{align}
\b{\phi}(\b{X}_t) \in \mathbb{R}^{p}.
\end{align}
In the WRF algorithm, \(\b{\phi}\) is constructed from sample-size and dimension information, pairwise-distance summaries, covariance statistics, PCA-spectrum statistics, and clustering-tendency scores.

\subsubsection{Sample Size and Dimension}
First, we include the sample size and dimension:
\begin{align}
&\phi_1(\b{X}_t) := \log(n_t), \\
&\phi_2(\b{X}_t) := d_t.
\end{align}
The logarithm is used for \(n_t\) because sample sizes may vary by orders of magnitude.

\subsubsection{Pairwise-Distance Distribution}
Next, we define the set of pairwise Euclidean distances in the standardized dataset:
\begin{align}
\mathcal{D}_t
:=
\left\{
\delta_{t,ij}
:=
\|\b{z}_{t,i}-\b{z}_{t,j}\|_2
\mid
1\leq i<j\leq n_t
\right\}.
\end{align}
From this set, we compute the mean, standard deviation, and selected quantiles:
\begin{align}
& m_t
:=
\frac{1}{|\mathcal{D}_t|}
\sum_{\delta\in\mathcal{D}_t} \delta, \\
& s_t
:=
\sqrt{
\frac{1}{|\mathcal{D}_t|}
\sum_{\delta\in\mathcal{D}_t}
(\delta-m_t)^2
}, \\
& q_{t,\alpha}
:=
\operatorname{Quantile}_{\alpha}(\mathcal{D}_t),
\end{align}
for $\alpha\in\{0.10,0.25,0.50,0.75,0.90\}$, where $|\mathcal{D}_t|$ denotes the size of set $\mathcal{D}_t$.
These terms summarize the global spread and distance distribution of the dataset.
Distance-based summaries are commonly used to characterize the geometry and dispersion of multivariate data \cite{mardia2024multivariate,bishop2006pattern}.

\subsubsection{Spread and Anisotropy}
We also compute the unbiased empirical covariance matrix of the standardized dataset $\b{Z}_t$, a standard second-order summary of multivariate data \cite{mardia2024multivariate,bishop2006pattern}:
\begin{align}
\b{C}_t
:=
\frac{1}{n_t-1}
\b{Z}_t^\top \b{Z}_t
\in \mathbb{R}^{d_t\times d_t}.
\end{align}
The total variance is measured by the trace of the covariance matrix:
\begin{align}
\operatorname{tr}(\b{C}_t),
\end{align}
and anisotropy is measured by the condition number:
\begin{align}
\kappa(\b{C}_t+\varepsilon\b{I})
:=
\frac{\lambda_{\max}(\b{C}_t+\varepsilon\b{I})}
{\lambda_{\min}(\b{C}_t+\varepsilon\b{I})},
\end{align}
where \(\lambda_{\max}\) and \(\lambda_{\min}\) denote the largest and smallest eigenvalues. A large condition number indicates that the dataset is elongated or anisotropic, while a condition number closer to one indicates a more isotropic distribution.

\subsubsection{Principal Variance Structure}
Let the normalized eigenvalues of \(\b{C}_t\) be:
\begin{align}
\rho_{t,j}
:=
\frac{\lambda_j(\b{C}_t)}
{\sum_{\ell=1}^{d_t}\lambda_\ell(\b{C}_t)+\varepsilon},
\end{align}
for $j \in \{1, \dots, d_t\}$, where:
\begin{align}
\rho_{t,1}\geq \rho_{t,2}\geq \cdots \geq \rho_{t,d_t}\geq 0.
\end{align}
These are the principal variance ratios, as in Principal Component Analysis (PCA) \cite{jolliffe2016principal,ghojogh2023principal}. We retain the top five values, padding with zeros if \(d_t<5\):
\begin{align}
\b{\rho}^{(5)}_t
:=
[\rho_{t,1},\ldots,\rho_{t,5}]^\top
\in \mathbb{R}^{5}.
\end{align}

\subsubsection{Clustering-Tendency Statistics}
Finally, we include clustering-tendency statistics. For each \(k \in \{2,3,4,5\}\), we apply \(k\)-means clustering \cite{mcqueen1967some,lloyd1982least} to \(Z_t\), producing cluster assignments:
\begin{align}
c^{(k)}_{t,i}\in\{1,\ldots,k\},
\end{align}
where \(c^{(k)}_{t,i}\) denotes the cluster label assigned to standardized data instance \(\b{z}_{t,i}\).
We calculate the corresponding silhouette score $ \operatorname{sil}_k(\b{Z}_t) $ as explained in the following. 

For a fixed \(k\), let:
\begin{align}
\mathcal{C}^{(k)}_{t,r}
:=
\Big\{
i \in \{1,\ldots,n_t\}
\,\big|\,
c^{(k)}_{t,i}=r
\Big\},
\end{align}
denote the index set of the \(r\)-th cluster. For each instance \(\b{z}_{t,i}\), we define its mean intra-cluster distance by:
\begin{align}
a^{(k)}_{t,i}
:=
\frac{1}{|\mathcal{C}^{(k)}_{t,c^{(k)}_{t,i}}|-1}
\sum_{\substack{j\in \mathcal{C}^{(k)}_{t,c^{(k)}_{t,i}}\\ j\neq i}}
\|\b{z}_{t,i}-\b{z}_{t,j}\|_2,
\end{align}
whenever the cluster containing \(\b{z}_{t,i}\) has more than one instance. This quantity measures how close \(\b{z}_{t,i}\) is to the other instances in its own cluster.

Next, for every cluster \(r\neq c^{(k)}_{t,i}\), we define the average distance from \(\b{z}_{t,i}\) to that cluster:
\begin{align}
d^{(k)}_{t,i}(r)
:=
\frac{1}{|\mathcal{C}^{(k)}_{t,r}|}
\sum_{j\in \mathcal{C}^{(k)}_{t,r}}
\|\b{z}_{t,i}-\b{z}_{t,j}\|_2.
\end{align}
The nearest-cluster distance of \(\b{z}_{t,i}\) is then:
\begin{align}
b^{(k)}_{t,i}
:=
\min_{r\neq c^{(k)}_{t,i}}
d^{(k)}_{t,i}(r).
\end{align}
This quantity measures how far \(\b{z}_{t,i}\) is from its nearest neighboring cluster.

The silhouette coefficient of \(\b{z}_{t,i}\) is defined as:
\begin{align}
s^{(k)}_{t,i}
:=
\frac{
b^{(k)}_{t,i}-a^{(k)}_{t,i}
}{
\max\left\{a^{(k)}_{t,i}, b^{(k)}_{t,i}\right\}
}.
\end{align}
Thus:
\begin{align}
-1 \leq s^{(k)}_{t,i} \leq 1.
\end{align}
A value close to \(1\) indicates that \(\b{z}_{t,i}\) is much closer to instances in its own cluster than to instances in other clusters. A value near \(0\) indicates that the instance lies near a cluster boundary. A negative value indicates that the instance may be closer, on average, to another cluster than to its assigned cluster.

The silhouette score of the whole dataset for \(k\) clusters is the average silhouette coefficient \cite{rousseeuw1987silhouettes}:
\begin{align}
\operatorname{sil}_k(\b{Z}_t)
:=
\frac{1}{n_t}
\sum_{i=1}^{n_t}
s^{(k)}_{t,i}.
\end{align}
This score summarizes how well the \(k\)-means partition separates the standardized dataset \(\b{Z}_t\) into \(k\) coherent clusters. If the silhouette score is not well-defined, for example because a cluster contains only one instance or because the clustering degenerates, we set it to zero.

\subsubsection{Unsupervised Descriptor}
Putting these quantities together, the unsupervised descriptor is:
\begin{align}
\b{\phi}(\b{X}_t)
:=
\begin{bmatrix}
\log(n_t)\\
d_t\\
m_t\\
s_t\\
q_{t,0.10}\\
q_{t,0.25}\\
q_{t,0.50}\\
q_{t,0.75}\\
q_{t,0.90}\\
\operatorname{tr}(\b{C}_t)\\
\kappa(\b{C}_t+\varepsilon\b{I})\\
\rho_{t,1}\\
\rho_{t,2}\\
\rho_{t,3}\\
\rho_{t,4}\\
\rho_{t,5}\\
\operatorname{sil}_2(\b{Z}_t)\\
\operatorname{sil}_3(\b{Z}_t)\\
\operatorname{sil}_4(\b{Z}_t)\\
\operatorname{sil}_5(\b{Z}_t)
\end{bmatrix}
\in \mathbb{R}^{20}.
\end{align}
Therefore, in the current variant of the WRF algorithm, we have $p=20$.

\subsection{Supervised Data Descriptor}

In the supervised case, each dataset $\b{X}_t$ is equipped with labels:
\begin{align}
\b{y}_t
\in
\{1, 2, \dots, K_t\}^{n_t},
\end{align}
where $K_t$ is the number of classes in the dataset $\b{X}_t$.

Each dataset is mapped to a $q$-dimensional supervised descriptor vector:
\begin{align}
\b{\psi}(\b{X}_t, \b{y}_t) \in \mathbb{R}^{q}.
\end{align}
In our implementation, \(\b{\psi}\) is constructed from number of classes, class entropy, class imbalance, within-class compactness, between-class separation, and class separability.

\subsubsection{Number of Classes}
Let the set of classes appearing in \(\b{y}_t\) be:
\begin{align}
\mathcal{C}_t
:=
\{c \mid y_{t,i}=c \text{ for at least one } i\}.
\end{align}
The number of classes for the dataset $\b{X}_t$ is:
\begin{align}
K_t := |\mathcal{C}_t|.
\end{align}

\subsubsection{Class Entropy}
For each class \(c\in\mathcal{C}_t\), we define the class index set:
\begin{align}
I_{t,c}
:=
\{i \mid y_{t,i}=c\},
\end{align}
and the class proportion:
\begin{align}
\pi_{t,c}
:=
\frac{|I_{t,c}|}{n_t}.
\end{align}
The class entropy for the $t$-th dataset is:
\begin{align}
H_t
:=
-\sum_{c\in\mathcal{C}_t}
\pi_{t,c}\log(\pi_{t,c}+\varepsilon),
\end{align}

\subsubsection{Class Imbalance}
The class imbalance for the $t$-th dataset is:
\begin{align}
\Delta_t
:=
\max_{c\in\mathcal{C}_t}\pi_{t,c}
-
\min_{c\in\mathcal{C}_t}\pi_{t,c}.
\end{align}

\subsubsection{Within-Class Compactness}
For each class \(c\), we define the standardized class subset:
\begin{align}
\b{Z}_{t,c}
:=
\{\b{z}_{t,i} \mid i\in I_{t,c}\}.
\end{align}
The class centroid is:
\begin{align}
\b{a}_{t,c}
:=
\frac{1}{|I_{t,c}|}
\sum_{i\in I_{t,c}}
\b{z}_{t,i} \in \mathbb{R}^{d_t}.
\end{align}
For each class, we define the summation of within-class distances for the $t$-th dataset:
\begin{align}
w_{t,c}
:=
\frac{2}{|I_{t,c}|(|I_{t,c}|-1)}
\sum_{\substack{i,j\in I_{t,c}\\ i<j}}
\|\b{z}_{t,i}-\b{z}_{t,j}\|_2,
\end{align}
when \(|I_{t,c}|\geq 2\). The mean within-class distance for the $t$-th dataset is:
\begin{align}
W_t
:=
\frac{1}{|\mathcal{C}_t^{\geq 2}|}
\sum_{c\in\mathcal{C}_t^{\geq 2}}
w_{t,c},
\end{align}
where:
\begin{align}
\mathcal{C}_t^{\geq 2} := \{c \in \mathcal{C}_t : |I_{t,c}| \geq 2 \}.
\end{align}
If $\mathcal{C}_t^{\geq 2}$ is an empty set, we set $W_t = 0$.


\subsubsection{Between-Class Separation}
The mean between-class distance for the $t$-th dataset is defined using the class centroids:
\begin{align}
B_t
:=
\frac{2}{K_t(K_t-1)}
\sum_{\substack{\{c,c'\}\subset\mathcal{C}_t, c \neq c'}}
\|\b{a}_{t,c}-\b{a}_{t,c'}\|_2,
\end{align}
when \(K_t\geq 2\), where the sum is taken over all unordered pairs of distinct classes. If \(K_t<2\), we set \(B_t=0\).

\subsubsection{Class Separability}
Finally, the class separability score is:
\begin{align}
S_t
:=
\frac{B_t}{W_t+\varepsilon},
\end{align}
where $\varepsilon >0$ is a small positive value for stability. 
A large value of \(S_t\) indicates that class centroids are far apart relative to the within-class spread.

\subsubsection{Supervised Descriptor}
The supervised descriptor is therefore:
\begin{align}
\b{\psi}(\b{X}_t,\b{y}_t)
:=
\begin{bmatrix}
K_t\\
H_t\\
\Delta_t\\
W_t\\
B_t\\
S_t
\end{bmatrix}
\in \mathbb{R}^{6}.
\end{align}
Thus, in the current variant of the WRF algorithm, there is $q=6$.

\subsection{Interpretation of Data Descriptors}

The unsupervised descriptor \(\b{\phi}(\b{X}_t)\) represents the unsupervised geometry of the dataset: sample size, dimension, distance distribution, spread, anisotropy, principal variance structure, and clustering tendency. The supervised descriptor \(\b{\psi}(\b{X}_t,\b{y}_t)\) represents the supervised organization of the dataset: number of classes, class entropy, class imbalance, within-class compactness, between-class separation, and class separability.

These descriptors are not meant to uniquely identify a dataset. Rather, they provide a structural summary through which different datasets can be compared even when they have different numbers of samples or no pointwise correspondence. The rule-following generator then attempts to produce a new dataset whose descriptors continue the pattern expressed by $\b{\phi}(\b{X}_1),\ldots,\b{\phi}(\b{X}_T)$, and in the supervised case, by $\b{\phi}(\b{X}_1),\b{\psi}(\b{X}_1,\b{y}_1),\ldots,
\b{\phi}(\b{X}_T),\b{\psi}(\b{X}_T,\b{y}_T)$.

\section{Rule Prediction in Descriptor Space}\label{section_rule_prediction}

The descriptor of the evolved dataset is predicted by extrapolating the descriptor sequence. 
We define both a rule-following target descriptor and a family-resemblance target descriptor, inspired by Wittgenstein's rule-following and Wittgenstein's family-resemblance, respectively. 

\subsection{Rule-Following Target Descriptor}

When \(T \geq 3\), we define:
\begin{align}
\b{f}_t := \b{\phi}(\b{X}_t) \in \mathbb{R}^p,
\end{align}
and we set:
\begin{align}
& \b{\Delta f}_T := \b{f}_T - \b{f}_{T-1}, \\
& \b{\Delta f}_{T-1} := \b{f}_{T-1} - \b{f}_{T-2}.
\end{align}
The next rule-following target descriptor is obtained by second-order extrapolation:
\begin{align}
\b{f}_{\mathrm{rule}}
:=
\b{f}_T + \b{\Delta f}_T + \frac{1}{2}(\b{\Delta f}_T-\b{\Delta f}_{T-1}).
\end{align}

Similarly, in the supervised case, we define:
\begin{align}
\b{g}_t := \b{\psi}(\b{X}_t,\b{y}_t) \in \mathbb{R}^q,
\end{align}
and we set:
\begin{align}
& \b{\Delta g}_T := \b{g}_T - \b{g}_{T-1}, \\
& \b{\Delta g}_{T-1} := \b{g}_{T-1} - \b{g}_{T-2}.
\end{align}
We compute an analogous supervised rule target by second-order extrapolation:
\begin{align}
\b{g}_{\mathrm{rule}}
:=
\b{g}_T + \b{\Delta g}_T + \frac{1}{2}(\b{\Delta g}_T-\b{\Delta g}_{T-1}).
\end{align}

The second-order version of WRF assumes $T \geq 3$. If $T=2$, first-order extrapolation is used. If $T=1$, we set $\b{f}_{\mathrm{rule}} = \b{f}_1$ and $\b{g}_{\mathrm{rule}} = \b{g}_1$.

\subsection{Family-Resemblance Target Descriptor}

In addition to rule-following, we define a family-resemblance target descriptor by averaging the history:
\begin{align}
&\b{f}_{\mathrm{fam}}
:=
\sum_{t=1}^{T} \alpha_t\, \b{\phi}(\b{X}_t), \\
&\sum_{t=1}^{T}\alpha_t = 1,
\quad
\alpha_t \geq 0,
\end{align}
where $\alpha_t \in [0,1]$ is the contributing weight of the $t$-th dataset. 
Equal weighting may be used so that all previous datasets contribute equally. 

The supervised family descriptor is similarly defined as:
\begin{align}
\b{g}_{\mathrm{fam}}
:=
\sum_{t=1}^{T} \alpha_t\, \b{\psi}(\b{X}_t,\b{y}_t).
\end{align}

\section{History-based Candidate Generation}\label{section_candidate_generation}

\subsection{Shape Prediction}

The shape of the $t$-th dataset is $(n_t,d_t)$ where $n_t$ and $d_t$ are the sample size and dimensionality of $\b{X}_t$, respectively.
The shape, i.e., the sample size and dimensionality, of the next dataset are predicted from the previous sample sizes and dimensions. 

We define:
\begin{align}
& \Delta n_T := n_T - n_{T-1}, \\
& \Delta n_{T-1} := n_{T-1} - n_{T-2}.
\end{align}
For \(T \geq 3\), we use second-order extrapolation to compute the sample size of the evolved dataset:
\begin{align}
&\widehat{n}_{T+1}
:= \nonumber
\\
&\max \Big\{1, \text{round}\big(n_T + \Delta n_T + \frac{1}{2}(\Delta n_T-\Delta n_{T-1}) \big)\Big\}.
\end{align}

Likewise, we define:
\begin{align}
& \Delta d_T := d_T - d_{T-1}, \\
& \Delta d_{T-1} := d_{T-1} - d_{T-2}.
\end{align}
We use second-order extrapolation to obtain the dimensionality of the evolved dataset:
\begin{align}
&\widehat{d}_{T+1}
:= \nonumber
\\
&\max \Big\{1, \text{round}\big(d_T + \Delta d_T + \frac{1}{2}(\Delta d_T-\Delta d_{T-1}) \big)\Big\}.
\end{align}

The second-order version of WRF assumes $T \geq 3$. If $T=2$, first-order extrapolation is used. If $T=1$, the sample size and dimensionality of the dataset are used without extrapolation.

\subsection{Mapping Datasets to the Target Dimensionality}

Since the datasets may have different dimensions, we first map each dataset to the predicted target dimension \(\widehat{d}_{T+1}\). We define the map:
\begin{align}
\mathcal{A}_{t\to \widehat{d}_{T+1}}:
\mathbb{R}^{n_t\times d_t}
\to
\mathbb{R}^{n_t\times \widehat{d}_{T+1}}
\end{align}
and set:
\begin{align}
\widetilde{\b{X}}_t
:=
\mathcal{A}_{t\to \widehat{d}_{T+1}}(\b{X}_t).
\end{align}
If \(d_t=\widehat{d}_{T+1}\), then \(\widetilde{\b{X}}_t=\b{X}_t\). If \(d_t>\widehat{d}_{T+1}\), we reduce the dimension, e.g., by PCA \cite{jolliffe2016principal,ghojogh2023principal}. If \(d_t<\widehat{d}_{T+1}\), we append small random coordinates to reach the target dimension.

\subsection{Candidate Generation}

We generate a finite candidate set containing candidate evolved datasets:
\begin{align}
\mathcal{C}
=
\{\b{\Gamma}^{(1)},\ldots,\b{\Gamma}^{(M)}\},
\end{align}
using only the observed datasets \(\{\b{X}_1,\ldots,\b{X}_T\}\). 
Note that the candidate generator does not use any external dataset identities. Instead, candidates are constructed from the history using the following mechanisms.

The candidate datasets $\{\b{\Gamma}^{(1)},\ldots,\b{\Gamma}^{(M)}\}$ can be constructed by different approaches, such as balanced family candidate generation and bounded mixture candidate generation. These two approaches are explained in the following. Alternatively, other approaches can be used to generate candidates. 

\subsubsection{Balanced Family Candidate Generation}

In one approach, the balanced family candidates can be generated by sampling approximately equal numbers of points from all previous datasets. 

We randomly sample from each mapped dataset:
\begin{align}
\b{\Gamma}_t \subseteq \widetilde{\b{X}}_t,
\qquad
\forall t \in \{1,\ldots,T\},
\end{align}
with approximately equal sample sizes per dataset, i.e., the size of each $\b{\Gamma}_t$ is approximately:
\begin{align}
|\b{\Gamma}_t| \approx \frac{\widehat{n}_{T+1}}{T}.
\end{align}
The balanced family candidate is obtained by putting together all the sampled datasets and shuffling them:
\begin{align}
\b{\Gamma}
=
\operatorname{Shuffle}
\left(
\b{\Gamma}_1\cup\b{\Gamma}_2\cup\cdots\cup\b{\Gamma}_T
\right)
\in
\mathbb{R}^{\widehat{n}_{T+1}\times \widehat{d}_{T+1}}.
\end{align}

In the supervised case, the sampling within each \(\widetilde{\b{X}}_t\) is performed in a class-balanced way, so that labels present in \(\b{y}_t\) are represented as evenly as possible.

\subsubsection{Bounded Mixture Candidate Generation}
Alternatively, bounded mixture candidates can be generated by sampling from all previous datasets using random but lower-bounded mixture proportions.

We construct the mixture weights in two steps. First, each previous dataset receives a guaranteed minimum contribution \(\pi_{\min} \in (0,1)\), e.g., we can use $\pi_{\min} = 0.25$ when $T = 4$. This ensures that no previous dataset is ignored. After assigning this minimum contribution to all \(T\) datasets, the remaining mass is $1-T\pi_{\min}$.
We then distribute this remaining mass randomly among the \(T\) datasets. Therefore, each final weight has the form:
\begin{align}
\pi_t
=
\pi_{\min}
+
\text{random extra assigned to dataset } t.
\end{align}
Hence, every dataset contributes at least \(\pi_{\min}\), while the remaining contribution is random. This gives random mixture candidates without allowing one previous dataset, especially the most recent one, to dominate the candidate.

To formalize this construction, we first sample from a Dirichlet distribution \cite{bishop2006pattern,murphy2012machine} because we need a random vector of proportions whose entries are nonnegative and sum to one:
\begin{align}
\b{u}
=
[u_1,\ldots,u_T]^\top
\sim
\operatorname{Dirichlet}(1,\ldots,1).
\end{align}

Let \(\pi_{\min}\) be the minimum allowed contribution from each previous dataset. We define the vector of mixture weights:
\begin{align}
\b{\pi}
= [\pi_1,\ldots,\pi_T]^\top =
\pi_{\min}\b{1}
+
(1-T\pi_{\min})\b{u},
\end{align}
where we require $0 \leq \pi_{\min} \leq 1/T$.
Therefore:
\begin{align}
\pi_t \geq \pi_{\min},
\qquad
\sum_{t=1}^{T}\pi_t=1.
\end{align}
The number of samples taken from each previous dataset is then drawn from a multinomial distribution \cite{bishop2006pattern,murphy2012machine} as:
\begin{align}
(n_1,\ldots,n_T)
\sim
\operatorname{Multinomial}
\left(
\widehat{n}_{T+1},
\b{\pi}
\right).
\end{align}

For each previous dataset, after aligning it to the predicted target dimension, we sample:
\begin{align}
\b{\Gamma}_t \subseteq \widetilde{\b{X}}_t,
\qquad
\forall t \in \{1,\ldots,T\}.
\end{align}
The subset \(\b{\Gamma}_t\) sampled from the aligned dataset
\(\widetilde{\b{X}}_t\) has size approximately:
\begin{align}
|\b{\Gamma}_t|
\approx
\pi_t\, \widehat{n}_{T+1}.
\end{align}

The candidate dataset is then obtained by stacking and shuffling the sampled subsets:
\begin{align}
\b{\Gamma}
=
\operatorname{Shuffle}
\left(
\b{\Gamma}_1\cup\b{\Gamma}_2\cup\cdots\cup\b{\Gamma}_T
\right)
\in
\mathbb{R}^{\widehat{n}_{T+1}\times \widehat{d}_{T+1}}.
\end{align}

Unlike the balanced family candidate, the mixture proportions do not need to be equal. However, the lower bound \(\pi_t\geq \pi_{\min}\) ensures that every previous dataset contributes a non-negligible number of samples. This prevents the candidate from being dominated by a single dataset, especially the most recent dataset \(\b{X}_T\).

\subsubsection{Labels of Candidate Datasets in Supervised Case}
In the supervised case, the sampled candidate datasets have their labels \(\{\b{\xi}_1,\ldots,\b{\xi}_M\}\), so the generated candidate datasets are tuples of data and labels:
\begin{align}
\mathcal{C}
=
\{(\b{\Gamma}^{(1)}, \b{\xi}^{(1)}),\ldots,(\b{\Gamma}^{(M)}, \b{\xi}^{(M)})\},
\end{align}
where $\b{\xi}^{(t)}$ denotes the labels of the candidate dataset $\b{\Gamma}^{(t)}$.

\subsection{Noise Injection and Shape Jitter}

\subsubsection{Noise Injection}

After candidates are generated using either balanced family generation or bounded mixture generation, a small Gaussian noise is added to each candidate so that generated datasets are not merely exact resamplings of previous data. 
\begin{align}
\b{\Gamma} \leftarrow \b{\Gamma} + \b{\epsilon},
\qquad
\b{\epsilon} \sim \mathcal{N}(\b{0},\sigma^2 \b{I}),
\end{align}
where $\sigma^2$ is the variance of noise and $\b{I}$ denotes the identity matrix. 
Noise injection is also a common strategy in data augmentation and regularization \cite{bishop1995training,chen2023unified}.

\subsubsection{Shape Jitter (Optional)}

Although the predicted next matrix shape is $(\widehat{n}_{T+1},\widehat{d}_{T+1})$, we may optionally allow candidate datasets to have small deviations from this predicted shape. That is, a candidate:
\begin{align}
\b{\Gamma}\in\mathbb{R}^{n(\b{\Gamma})\times d(\b{\Gamma})},
\end{align}
may satisfy:
\begin{align}
n(\b{\Gamma})\neq \widehat{n}_{T+1}
\qquad\text{or}\qquad
d(\b{\Gamma})\neq \widehat{d}_{T+1}.
\end{align}
We call this optional variation shape jitter. It increases diversity in the candidate pool by allowing the algorithm to explore nearby sample sizes and ambient dimensions. 

\section{Candidate Scoring}\label{section_candidate_scoring}

Each candidate \(\b{\Gamma} \in \mathcal{C}\) in unsupervised case, or  \((\b{\Gamma}, \b{\xi}) \in \mathcal{C}\) in supervised case, is scored by a weighted loss. 

\subsection{Unsupervised Scores}

The unsupervised components of the loss are:
\begin{align}
&\mathcal{L}_{\mathrm{rule}}(\b{\Gamma})
:=
D\big(\b{\phi}(\b{\Gamma}), f_{\mathrm{rule}}\big), \\
&\mathcal{L}_{\mathrm{fam}}(\b{\Gamma})
:=
D\big(\b{\phi}(\b{\Gamma}), f_{\mathrm{fam}}\big), \\
&\mathcal{L}_{\mathrm{last}}(\b{\Gamma})
:=
D\big(\b{\phi}(\b{\Gamma}), \b{\phi}(\b{X}_T)\big),
\end{align}
where \(D(\cdot,\cdot)\) is a scale-normalized descriptor distance, defined as:
\begin{align}\label{equation_D_a_b_loss}
D(\b{a},\b{b})
:=
\frac{1}{r}
\sum_{\ell=1}^{r}
\left(
\frac{a_\ell-b_\ell}
{|a_\ell|+|b_\ell|+\varepsilon}
\right)^2,
\end{align}
for two descriptor vectors \(\b{a},\b{b}\in\mathbb{R}^r\), where $\varepsilon$ is a small positive number for stability of division, $a_\ell$ and $b_\ell$ denote the $\ell$-th element of vector, and we have \(r=p\) for \(\b{\phi}\)-descriptors.

When using shape jitter, the candidates whose matrix shapes deviate too far from the predicted shape should be penalized. Therefore, we include the shape loss:
\begin{align}
\mathcal{L}_{\mathrm{shape}}(\b{\Gamma})
:=
\left(
\frac{n(\b{\Gamma})-\widehat{n}_{T+1}}{\widehat{n}_{T+1}}
\right)^2
+
\left(
\frac{d(\b{\Gamma})-\widehat{d}_{T+1}}{\widehat{d}_{T+1}}
\right)^2.
\end{align}
When shape jitter is not used, all candidates are constructed with $n(\b{\Gamma})=\widehat{n}_{T+1}$ and $d(\b{\Gamma})=\widehat{d}_{T+1}$, and hence $\mathcal{L}_{\mathrm{shape}}(\b{\Gamma})=0$.

We also include a non-collapse penalty to make sure that the evolved dataset does not collapse to a single point:
\begin{align}
\mathcal{L}_{\mathrm{col}}(\b{\Gamma})
:=
\frac{1}{\operatorname{std}\{ \|\b{\gamma}_i-\b{\gamma}_j\|_2 \mid i<j\}+\varepsilon},
\end{align}
where $\operatorname{std}(.)$ is the standard deviation operator, $\b{\gamma}_i$ is the $i$-th data instance in the dataset $\b{\Gamma}$, and $\|.\|_2$ denotes the $\ell_2$ norm, and $\varepsilon$ is a small positive number for stability of division.

\subsection{Supervised Scores}

In the supervised case, the additional label-aware losses are:
\begin{align}
&\mathcal{L}_{\mathrm{sup\text{-}rule}}(\b{\Gamma})
:=
D\big(\b{\psi}(\b{\Gamma},\b{\xi}), \b{g}_{\mathrm{rule}}\big), \\
&\mathcal{L}_{\mathrm{sup\text{-}fam}}(\b{\Gamma})
:=
D\big(\b{\psi}(\b{\Gamma},\b{\xi}), \b{g}_{\mathrm{fam}}\big), \\
&\mathcal{L}_{\mathrm{sup\text{-}last}}(\b{\Gamma})
:=
D\big(\b{\psi}(\b{\Gamma},\b{\xi}), \b{\psi}(\b{X}_T,\b{y}_T)\big),
\end{align}
where the loss function $D(.,.)$ is defined in Eq. (\ref{equation_D_a_b_loss}), in which we have \(r=q\) for \(\b{\psi}\)-descriptors.

\subsection{Normalizing the Scores}

Since different loss terms can have different numerical scales, each loss is normalized over the candidate pool. If \(\mathcal{L}_k\) is one of the loss components described above, its normalized version is:
\begin{align}
\widetilde{\mathcal{L}}_k(\b{\Gamma})
=
\operatorname{clip}
\left(
\frac{\mathcal{L}_k(\b{\Gamma})-Q_{0.05}(\mathcal{L}_k)}
     {Q_{0.95}(\mathcal{L}_k)-Q_{0.05}(\mathcal{L}_k)+\varepsilon},
0,1
\right).
\end{align}
Here, \(Q_{0.05}\) and \(Q_{0.95}\) are the empirical 5th and 95th percentiles over the candidate pool. The constant $\varepsilon$ is a small positive number for stability of division, and $\operatorname{clip}(\cdot,0,1)$ restricts the value to the range $[0,1]$.

\subsection{Total Candidate Score}

The total candidate loss score is then:
\begin{align}
\mathcal{S}(\b{\Gamma})
=
\frac{
\sum_k \lambda_k \widetilde{\mathcal{L}}_k(\b{\Gamma})
}{
\sum_k \lambda_k + \varepsilon
},
\end{align}
where $\varepsilon$ is a small positive number for stability of division and:
\begin{align}
\lambda_k \in [0,1],
\end{align}
are user-specified importance weights for the loss functions.

In the unsupervised cases, we have $k=5$ and the loss terms $\mathcal{L}_{\mathrm{rule}}$, $\mathcal{L}_{\mathrm{fam}}$, $\mathcal{L}_{\mathrm{last}}$, $\mathcal{L}_{\mathrm{shape}}$, and $\mathcal{L}_{\mathrm{col}}$ are used.
In the supervised cases, we have $k=8$ and the loss terms $\mathcal{L}_{\mathrm{rule}}$, $\mathcal{L}_{\mathrm{fam}}$, $\mathcal{L}_{\mathrm{last}}$, $\mathcal{L}_{\mathrm{shape}}$, $\mathcal{L}_{\mathrm{col}}$, $\mathcal{L}_{\mathrm{sup\text{-}rule}}$, $\mathcal{L}_{\mathrm{sup\text{-}fam}}$, and $\mathcal{L}_{\mathrm{sup\text{-}last}}$ are used.


\subsection{Selection of Best Candidate Dataset}

After the candidate scoring step, we select the best candidate dataset:
\begin{align}
\b{\Gamma}^{(0)}
=
\arg\min_{\b{\Gamma}\in \mathcal{C}}
S(\b{\Gamma}),
\end{align}
where \(S(\b{\Gamma})\) is the candidate score defined in Eq.~(93). In the supervised case, the selected candidate also has an associated label vector, denoted by \(\b{\xi}^{(0)}\). The candidate \(\b{\Gamma}^{(0)}\) is the generated dataset before refinement.

\section{Data Refinement}\label{section_data_refinement}

The purpose of the data refinement step is to slightly adjust the coordinates of the generated data instances so that the evolved dataset better matches the rule-following and family-resemblance targets. During this step, only the data coordinates are optimized. The labels \(\b{\xi}^{(0)}\), if available, are kept fixed. Thus, the optimization variable is a dataset:
\begin{align}
\b{\Gamma}
\in
\mathbb{R}^{n(\b{\Gamma}^{(0)})\times d(\b{\Gamma}^{(0)})},
\end{align}
initialized by:
\begin{align}
\b{\Gamma} = \b{\Gamma}^{(0)}.
\end{align}

The full descriptors \(\b{\phi}\) and \(\b{\psi}\), defined in Section \ref{section_dataset_descriptors}, contain some terms which are not convenient for gradient-based optimization. For example, \(\b{\phi}\) contains empirical quantiles, \(k\)-means cluster assignments, and silhouette scores. These quantities are useful for scoring candidates, but they are not differentiable with respect to the coordinates of the generated data points. Therefore, in the refinement step, we use differentiable surrogate descriptors:
\begin{align}
\widetilde{\b{\phi}}(\b{\Gamma}),
\end{align}
and, in the supervised case:
\begin{align}
\widetilde{\b{\psi}}(\b{\Gamma},\b{\xi}).
\end{align}
This follows the general idea of differentiable optimization, where generated variables are adjusted using gradients of a differentiable objective \cite{baydin2018automatic}.

\subsection{Differentiable Unsupervised Descriptor}

The descriptor \(\widetilde{\b{\phi}}\) is a differentiable approximation of \(\b{\phi}\). It contains differentiable geometric statistics such as pairwise-distance summaries, covariance trace, and covariance-spectrum information. Similarly, \(\widetilde{\b{\psi}}\) is a differentiable approximation of \(\b{\psi}\). It contains differentiable supervised statistics such as class centroids, mean within-class distance, mean between-class distance, and class separability.

For example, let:
\begin{align}
\b{\Gamma}
=
\begin{bmatrix}
\b{\gamma}_1^\top\\
\vdots\\
\b{\gamma}_m^\top
\end{bmatrix}
\in
\mathbb{R}^{m\times d}.
\end{align}
We first standardize \(\b{\Gamma}\) feature-wise and denote the standardized dataset by \(\overline{\b{\Gamma}}\):
\begin{align}\label{equation_Gamma_bar}
\overline{\b{\Gamma}}
=
\begin{bmatrix}
\overline{\b{\gamma}}_1^\top\\
\vdots\\
\overline{\b{\gamma}}_m^\top
\end{bmatrix}.
\end{align}

Let:
\begin{align}
\mathcal{D}(\overline{\b{\Gamma}})
:=
\left\{
\|\overline{\b{\gamma}}_i-\overline{\b{\gamma}}_j\|_2
\mid
1\leq i<j\leq m
\right\},
\end{align}
be the set of pairwise distances in the standardized generated dataset. We may use the mean and standard deviation of this set:
\begin{align}
\mu_D(\b{\Gamma})
:=
\operatorname{mean}\big(\mathcal{D}(\overline{\b{\Gamma}})\big),
\end{align}
\begin{align}
\sigma_D(\b{\Gamma})
:=
\operatorname{std}\big(\mathcal{D}(\overline{\b{\Gamma}})\big).
\end{align}
Let the covariance matrix of \(\overline{\b{\Gamma}}\) be:
\begin{align}
\b{C}_{\Gamma}
:=
\frac{1}{m-1}
\overline{\b{\Gamma}}^{\top}
\overline{\b{\Gamma}}.
\end{align}
Let the normalized eigenvalues of \(\b{C}_{\Gamma}\) be:
\begin{align}
\rho_j(\b{\Gamma})
:=
\frac{
\lambda_j(\b{C}_{\Gamma})
}{
\sum_{\ell=1}^{d}\lambda_{\ell}(\b{C}_{\Gamma})+\varepsilon
},
\end{align}
where the eigenvalues are sorted in descending order. A possible differentiable unsupervised descriptor is then:
\begin{align}
\widetilde{\b{\phi}}(\b{\Gamma})
:=
\begin{bmatrix}
\mu_D(\b{\Gamma})\\
\sigma_D(\b{\Gamma})\\
\operatorname{tr}(\b{C}_{\Gamma})\\
\rho_1(\b{\Gamma})\\
\rho_2(\b{\Gamma})\\
\rho_3(\b{\Gamma})\\
\rho_4(\b{\Gamma})\\
\rho_5(\b{\Gamma})
\end{bmatrix} 
\in \mathbb{R}^8.
\end{align}
If \(d<5\), the missing eigenvalue entries are padded by zeros.

This descriptor $\widetilde{\b{\phi}}(\b{\Gamma})$ is a differentiable surrogate of the unsupervised descriptor \(\b{\phi}\). It does not contain all components of \(\b{\phi}\), because some entries of \(\b{\phi}\), such as empirical quantiles, \(k\)-means assignments, and silhouette scores, are not convenient for gradient-based optimization. Instead, \(\widetilde{\b{\phi}}\) keeps the differentiable geometric components of the descriptor: the mean and standard deviation of pairwise distances, the covariance trace, and the leading normalized covariance eigenvalues. All entries of \(\widetilde{\b{\phi}}(\b{\Gamma})\) depend on the coordinates of \(\b{\Gamma}\), and therefore all of them can contribute gradients during refinement.

We now define the refinement targets in the differentiable descriptor space. For every observed dataset \(\b{X}_t\), we define:
\begin{align}
\widetilde{\b{f}}_t
:=
\widetilde{\b{\phi}}(\b{X}_t),
\qquad
t\in\{1,\ldots,T\}.
\end{align}
The differentiable rule-following target is obtained by the same second-order extrapolation used before:
\begin{equation}
\begin{aligned}
\widetilde{\b{f}}_{\mathrm{rule}}
:=\,
&\widetilde{\b{f}}_T
+
(\widetilde{\b{f}}_T-\widetilde{\b{f}}_{T-1})
\\
&+
\frac{1}{2}
\left[
(\widetilde{\b{f}}_T-\widetilde{\b{f}}_{T-1})
-
(\widetilde{\b{f}}_{T-1}-\widetilde{\b{f}}_{T-2})
\right].
\end{aligned}
\end{equation}
for \(T\geq 3\).

The differentiable family-resemblance target is:
\begin{align}
\widetilde{\b{f}}_{\mathrm{fam}}
:=
\sum_{t=1}^{T}
\alpha_t
\widetilde{\b{f}}_t,
\end{align}
where:
\begin{align}
\sum_{t=1}^{T}\alpha_t = 1,
\qquad
\alpha_t\geq 0.
\end{align}
The differentiable last-dataset target is:
\begin{align}
\widetilde{\b{f}}_{\mathrm{last}}
:=
\widetilde{\b{f}}_T.
\end{align}

Using the descriptor distance \(D(\cdot,\cdot)\) defined in Eq. (\ref{equation_D_a_b_loss}), the unsupervised refinement objective is:
\begin{equation}\label{equation_J_unsup}
\begin{aligned}
\mathcal{J}_{\mathrm{unsup}}&(\b{\Gamma})
:= \\
&
\beta_{\mathrm{rule}}
D\big(
\widetilde{\b{\phi}}(\b{\Gamma}),
\widetilde{\b{f}}_{\mathrm{rule}}
\big)
+
\beta_{\mathrm{fam}}
D\big(
\widetilde{\b{\phi}}(\b{\Gamma}),
\widetilde{\b{f}}_{\mathrm{fam}}
\big)
\\
&
+
\beta_{\mathrm{last}}
D\big(
\widetilde{\b{\phi}}(\b{\Gamma}),
\widetilde{\b{f}}_{\mathrm{last}}
\big)
+
\beta_{\mathrm{col}}
\mathcal{L}_{\mathrm{col}}(\b{\Gamma}),
\end{aligned}
\end{equation}
where $\beta_{\mathrm{rule}}$, $\beta_{\mathrm{fam}}$, $\beta_{\mathrm{last}}$, and $\beta_{\mathrm{col}}$ are non-negative weights for the loss terms in $\mathcal{J}_{\mathrm{unsup}}(\b{\Gamma})$.
The collapse penalty in Eq. (\ref{equation_J_unsup}) is:
\begin{align}
\mathcal{L}_{\mathrm{col}}(\b{\Gamma})
:=
\frac{1}
{
\operatorname{std}\!\big(
\mathcal{D}(\overline{\b{\Gamma}})
\big)
+\varepsilon
},
\end{align}
where $\operatorname{std}(.)$ denotes the standard deviation.
This term discourages degenerate solutions in which the generated data instances collapse into an overly concentrated configuration.

\subsection{Differentiable Supervised Descriptor}

In the supervised case, the refinement step also uses a differentiable supervised descriptor $\widetilde{\b{\psi}}(\b{\Gamma},\b{\xi})$ where \(\b{\Gamma}\) is the candidate dataset and \(\b{\xi}\) is its fixed label vector. The labels are not optimized during refinement; only the coordinates of the data instances in \(\b{\Gamma}\) are optimized.

Let:
\begin{align}
\b{\Gamma}
=
\begin{bmatrix}
\b{\gamma}_1^\top\\
\vdots\\
\b{\gamma}_m^\top
\end{bmatrix}
\in
\mathbb{R}^{m\times d},
\qquad
\b{\xi}
=
[\xi_1,\ldots,\xi_m]^\top.
\end{align}
We standardize \(\b{\Gamma}\) feature-wise and denote the standardized dataset by $\overline{\b{\Gamma}}$ as in Eq. (\ref{equation_Gamma_bar}), with rows $\{\overline{\b{\gamma}}_1^\top, \dots, \overline{\b{\gamma}}_m^\top\}$.

Let the set of labels appearing in \(\b{\xi}\) be:
\begin{align}
\mathcal{C}_{\xi}
:=
\{c \mid \xi_i=c \text{ for at least one } i\},
\end{align}
and let:
\begin{align}
K_{\xi}
:=
|\mathcal{C}_{\xi}|,
\end{align}
be the number of classes. For each class \(c\in\mathcal{C}_{\xi}\), we define the class index set:
\begin{align}
I_{\xi,c}
:=
\{i \mid \xi_i=c\},
\end{align}
and the class proportion:
\begin{align}
\pi_{\xi,c}
:=
\frac{|I_{\xi,c}|}{m}.
\end{align}
The class entropy is:
\begin{align}
H_{\xi}
:=
-
\sum_{c\in\mathcal{C}_{\xi}}
\pi_{\xi,c}
\log(\pi_{\xi,c}+\varepsilon),
\end{align}
and the class imbalance is:
\begin{align}
\Delta_{\xi}
:=
\max_{c\in\mathcal{C}_{\xi}}\pi_{\xi,c}
-
\min_{c\in\mathcal{C}_{\xi}}\pi_{\xi,c}.
\end{align}
These two quantities depend only on the fixed labels \(\b{\xi}\), so they are constant with respect to the coordinates of \(\b{\Gamma}\).

For each class \(c\), we define the class centroid in the standardized candidate dataset:
\begin{align}
\widetilde{\b{a}}_{\xi,c}
:=
\frac{1}{|I_{\xi,c}|}
\sum_{i\in I_{\xi,c}}
\overline{\b{\gamma}}_i.
\end{align}
For every class with at least two instances, we define its mean within-class distance:
\begin{align}
\widetilde{w}_{\xi,c}
:=
\frac{2}{|I_{\xi,c}|(|I_{\xi,c}|-1)}
\sum_{\substack{i,j\in I_{\xi,c}\\ i<j}}
\|\overline{\b{\gamma}}_i-\overline{\b{\gamma}}_j\|_2.
\end{align}
The differentiable mean within-class distance is then:
\begin{align}
\widetilde{W}_{\xi}
:=
\frac{1}{|\mathcal{C}_{\xi}^{\geq 2}|}
\sum_{c\in\mathcal{C}_{\xi}^{\geq 2}}
\widetilde{w}_{\xi,c},
\end{align}
where:
\begin{align}
\mathcal{C}_{\xi}^{\geq 2}
:=
\{c\in\mathcal{C}_{\xi}: |I_{\xi,c}|\geq 2\}.
\end{align}
If no class has at least two instances, we set \(\widetilde{W}_{\xi}=0\).

The differentiable mean between-class distance is computed using the class centroids:
\begin{align}
\widetilde{B}_{\xi}
:=
\frac{2}{K_{\xi}(K_{\xi}-1)}
\sum_{\substack{\{c,c'\}\subset\mathcal{C}_{\xi}, c \neq c'}}
\|\widetilde{\b{a}}_{\xi,c}-\widetilde{\b{a}}_{\xi,c'}\|_2,
\end{align}
when \(K_{\xi}\geq 2\), where the sum is over unordered pairs of distinct classes. If \(K_{\xi}<2\), we set \(\widetilde{B}_{\xi}=0\).

The differentiable class separability score is:
\begin{align}
\widetilde{S}_{\xi}
:=
\frac{\widetilde{B}_{\xi}}
{\widetilde{W}_{\xi}+\varepsilon}.
\end{align}

Therefore, the differentiable supervised descriptor is:
\begin{align}
\widetilde{\b{\psi}}(\b{\Gamma},\b{\xi})
:=
\begin{bmatrix}
K_{\xi}\\
H_{\xi}\\
\Delta_{\xi}\\
\widetilde{W}_{\xi}\\
\widetilde{B}_{\xi}\\
\widetilde{S}_{\xi}
\end{bmatrix}
\in
\mathbb{R}^{6}.
\end{align}
This descriptor has the same components as the supervised descriptor \(\b{\psi}\), but it is written using differentiable operations with respect to the coordinates of \(\b{\Gamma}\). The first three entries, \(K_{\xi}\), \(H_{\xi}\), and \(\Delta_{\xi}\), are fixed during refinement because the label vector \(\b{\xi}\) is fixed. The last three entries, \(\widetilde{W}_{\xi}\), \(\widetilde{B}_{\xi}\), and \(\widetilde{S}_{\xi}\), depend on the coordinates of \(\b{\Gamma}\) and therefore contribute gradients during optimization.

We denote the differentiable supervised descriptors as:
\begin{align}
\widetilde{\b{g}}_t
:=
\widetilde{\b{\psi}}(\b{X}_t,\b{y}_t),
\qquad
t\in\{1,\ldots,T\}.
\end{align}
The corresponding supervised rule-following target is:
\begin{equation}
\begin{aligned}
\widetilde{\b{g}}_{\mathrm{rule}}
:=\,
&\widetilde{\b{g}}_T
+
(\widetilde{\b{g}}_T-\widetilde{\b{g}}_{T-1})
\\
&+
\frac{1}{2}
\left[
(\widetilde{\b{g}}_T-\widetilde{\b{g}}_{T-1})
-
(\widetilde{\b{g}}_{T-1}-\widetilde{\b{g}}_{T-2})
\right].
\end{aligned}
\end{equation}
The supervised family-resemblance target is:
\begin{align}
\widetilde{\b{g}}_{\mathrm{fam}}
:=
\sum_{t=1}^{T}
\alpha_t
\widetilde{\b{g}}_t,
\end{align}
and the supervised last-dataset target is:
\begin{align}
\widetilde{\b{g}}_{\mathrm{last}}
:=
\widetilde{\b{g}}_T.
\end{align}
The supervised refinement objective is:
\begin{align}
\mathcal{J}_{\mathrm{sup}}(\b{\Gamma},\b{\xi}^{(0)})
:=\,
&
\beta_{\mathrm{sup\text{-}rule}}
D\big(
\widetilde{\b{\psi}}(\b{\Gamma},\b{\xi}^{(0)}),
\widetilde{\b{g}}_{\mathrm{rule}}
\big)
\nonumber\\
&
+
\beta_{\mathrm{sup\text{-}fam}}
D\big(
\widetilde{\b{\psi}}(\b{\Gamma},\b{\xi}^{(0)}),
\widetilde{\b{g}}_{\mathrm{fam}}
\big)
\nonumber\\
&
+
\beta_{\mathrm{sup\text{-}last}}
D\big(
\widetilde{\b{\psi}}(\b{\Gamma},\b{\xi}^{(0)}),
\widetilde{\b{g}}_{\mathrm{last}}
\big).
\end{align}

\subsection{Full Differentiable Refinement Objective}

The full differentiable refinement objective is:
\begin{align}
\mathcal{J}(\b{\Gamma})
:=
\mathcal{J}_{\mathrm{unsup}}(\b{\Gamma})
+
\eta_{\mathrm{sup}}
\mathcal{J}_{\mathrm{sup}}(\b{\Gamma},\b{\xi}^{(0)}),
\end{align}
where:
\begin{align}
\eta_{\mathrm{sup}}
=
\begin{cases}
1, & \text{if supervised labels are used},\\
0, & \text{otherwise}.
\end{cases}
\end{align}

Starting from \(\b{\Gamma}^{(0)}\), we minimize:
\begin{align}
\min_{\b{\Gamma}}\,
\mathcal{J}(\b{\Gamma}),
\end{align}
by gradient-based optimization. 
It is possible to use the Adam optimizer, a first-order adaptive gradient method based on estimates of first and second moments of the gradients, for optimization \cite{kingma2014adam,ghojogh2026backpropagation}.
If \(\b{\Gamma}^{(r)}\) denotes the refined candidate at iteration \(r\), then:
\begin{align}
\b{\Gamma}^{(r+1)}
=
\operatorname{AdamStep}
\left(
\b{\Gamma}^{(r)},
\nabla_{\b{\Gamma}}
\mathcal{J}(\b{\Gamma}^{(r)})
\right).
\end{align}
Rather than using the last iterate $R$, we use the best iterate observed during optimization:
\begin{align}
r^{\star}
:=
\arg\min_{0\leq r\leq R}
\mathcal{J}(\b{\Gamma}^{(r)}),
\end{align}
and the final evolved dataset is:
\begin{align}
\widehat{\b{X}}_{T+1}
:=
\b{\Gamma}^{(r^{\star})}.
\end{align}
In the supervised case, its labels are:
\begin{align}
\widehat{\b{y}}_{T+1}
:=
\b{\xi}^{(0)}.
\end{align}

The refinement step therefore does not generate a new candidate from scratch. It only adjusts the coordinates of the selected candidate so that its differentiable descriptors better follow the rule and family-resemblance targets.

\section{Algorithm and Complexity Analysis}\label{section_algorithm_complexity}

In this section, we summarize the WRF data evolution method algorithmically and analyze its time and space complexities. The method consists of four main stages: descriptor computation and rule prediction, history-based candidate generation, candidate scoring, and optional differentiable refinement.

\subsection{Algorithm of WRF}

The full procedure of WRF data evolution is summarized in Algorithm~\ref{algorithm_WRF_main}. The input is a sequence of observed datasets
\[
\{\b{X}_t\}_{t=1}^{T},
\]
and, in the supervised case, a sequence of label vectors
\[
\{\b{y}_t\}_{t=1}^{T}.
\]
The output is an evolved dataset \(\widehat{\b{X}}_{T+1}\), and, in the supervised case, an evolved label vector \(\widehat{\b{y}}_{T+1}\).

Algorithm~\ref{algorithm_WRF_main} first computes the unsupervised descriptors
\[
\b{f}_t := \b{\phi}(\b{X}_t),
\qquad
t\in\{1,\ldots,T\},
\]
and, when labels are available, the supervised descriptors
\[
\b{g}_t := \b{\psi}(\b{X}_t,\b{y}_t).
\]
These descriptors summarize the structural properties of the observed datasets. The rule-following target descriptors \(\b{f}_{\mathrm{rule}}\) and \(\b{g}_{\mathrm{rule}}\) are then computed by second-order extrapolation in descriptor space. In parallel, the family-resemblance target descriptors \(\b{f}_{\mathrm{fam}}\) and \(\b{g}_{\mathrm{fam}}\) are computed as weighted averages of the historical descriptors. Thus, the algorithm uses both a rule-following principle and a family-resemblance principle.

After computing the target descriptors, Algorithm~\ref{algorithm_WRF_main} predicts the next matrix shape
\[
(\widehat{n}_{T+1},\widehat{d}_{T+1}),
\]
where \(\widehat{n}_{T+1}\) is the predicted number of samples and \(\widehat{d}_{T+1}\) is the predicted dimension of the evolved dataset. The algorithm then calls Algorithm~\ref{algorithm_WRF_candidate_generation} to construct a finite candidate pool
\[
\mathcal{C}
=
\{\b{\Gamma}^{(1)},\ldots,\b{\Gamma}^{(M)}\}.
\]
In the supervised case, the candidate pool consists of pairs:
\[
\mathcal{C}
=
\{(\b{\Gamma}^{(1)},\b{\xi}^{(1)}),\ldots,(\b{\Gamma}^{(M)},\b{\xi}^{(M)})\}.
\]

\SetAlCapSkip{0.5em}
\IncMargin{0.8em}
\begin{algorithm2e}[!t]
\DontPrintSemicolon
    \textbf{Input}: datasets $\{\b{X}_t\}_{t=1}^{T}$; optional labels $\{\b{y}_t\}_{t=1}^{T}$; number of candidates $M$.\;
    \textbf{Output}: evolved dataset $\widehat{\b{X}}_{T+1}$; optional evolved labels $\widehat{\b{y}}_{T+1}$.\;
    \;
    \For{$t=1,\ldots,T$}{
        Standardize $\b{X}_t$ feature-wise to obtain $\b{Z}_t$.\;
        Compute the unsupervised descriptor:\;
        $\qquad\b{f}_t := \b{\phi}(\b{X}_t).$\;
        \If{labels are available}{
            Compute the supervised descriptor:\;
            $\qquad\b{g}_t := \b{\psi}(\b{X}_t,\b{y}_t).$\;
        }
    }
    Compute the rule-following target descriptor:\;
    $\,\,\,\,\b{f}_{\mathrm{rule}}
    :=
    \b{f}_T
    +
    (\b{f}_T-\b{f}_{T-1})
    +
    \frac{1}{2}
    \left[
    (\b{f}_T-\b{f}_{T-1})
    -
    (\b{f}_{T-1}-\b{f}_{T-2})
    \right].$\;
    Compute the family-resemblance target descriptor: $\b{f}_{\mathrm{fam}}
    :=
    \sum_{t=1}^{T}\alpha_t\b{f}_t.$\;
    \If{labels are available}{
        Compute $\b{g}_{\mathrm{rule}}$ and $\b{g}_{\mathrm{fam}}$ analogously from $\{\b{g}_t\}_{t=1}^{T}$.\;
    }
    Predict the next matrix shape: $(\widehat{n}_{T+1},\widehat{d}_{T+1}).$\;
    Generate candidate pool $\mathcal{C}$ using Algorithm~\ref{algorithm_WRF_candidate_generation}.\;
    Score all candidates using Algorithm~\ref{algorithm_WRF_candidate_scoring}.\;
    Select the best candidate:\;
    $\qquad\b{\Gamma}^{(0)}
    :=
    \arg\min_{\b{\Gamma}\in\mathcal{C}}
    S(\b{\Gamma}).$\;
    \If{labels are available}{
        Let $\b{\xi}^{(0)}$ be the label vector associated with $\b{\Gamma}^{(0)}$.\;
    }
    \uIf{refinement is used}{
        Refine $\b{\Gamma}^{(0)}$ using Algorithm~\ref{algorithm_WRF_refinement} and set the result to $\widehat{\b{X}}_{T+1}$.\;
    }
    \Else{
        Set: $\widehat{\b{X}}_{T+1}:=\b{\Gamma}^{(0)}.$
    }
    \If{labels are available}{
        Set: $\widehat{\b{y}}_{T+1}:=\b{\xi}^{(0)}.$
    }
    \KwRet $\widehat{\b{X}}_{T+1}$ and, if supervised, $\widehat{\b{y}}_{T+1}$.\;

\caption{The main WRF algorithm for rule-following data evolution.}\label{algorithm_WRF_main}
\end{algorithm2e}
\DecMargin{0.8em}

Algorithm~\ref{algorithm_WRF_candidate_generation} generates candidates from the observed history. First, all previous datasets are mapped to the predicted target dimension \(\widehat{d}_{T+1}\). Then each candidate is generated by one of the history-based mechanisms. In balanced family generation, approximately equal numbers of samples are taken from all previous datasets. In bounded mixture generation, random mixture proportions are used, but each previous dataset receives at least a minimum contribution. This prevents a candidate from being dominated by only one historical dataset. After sampling, the selected subsets are stacked, shuffled, and perturbed by small Gaussian noise.

Once the candidate pool has been constructed, Algorithm~\ref{algorithm_WRF_candidate_scoring} scores all candidates. For each candidate \(\b{\Gamma}^{(m)}\), the algorithm computes the rule-following loss, family-resemblance loss, last-dataset loss, shape loss, and collapse loss. In the supervised case, it also computes the supervised rule-following, supervised family-resemblance, and supervised last-dataset losses. Since these loss terms may have different numerical scales, each loss type is normalized over the candidate pool. The total score
\[
S(\b{\Gamma}^{(m)})
\]
is then computed as a weighted average of the normalized losses. The selected candidate is:
\[
\b{\Gamma}^{(0)}
:=
\arg\min_{\b{\Gamma}^{(m)}\in\mathcal{C}}
S(\b{\Gamma}^{(m)}).
\]

Finally, if refinement is enabled, Algorithm~\ref{algorithm_WRF_refinement} refines the selected candidate by gradient-based optimization. This step does not generate a new candidate from scratch. Rather, it starts from \(\b{\Gamma}^{(0)}\) and adjusts its coordinates so that its differentiable descriptors better match the rule-following and family-resemblance targets. The labels, if available, are kept fixed during refinement. The best refinement iterate is returned as the final evolved dataset:
\[
\widehat{\b{X}}_{T+1}
=
\b{\Gamma}^{(r^\star)}.
\]
If refinement is not used, the final evolved dataset is simply:
\[
\widehat{\b{X}}_{T+1}
=
\b{\Gamma}^{(0)}.
\]

\SetAlCapSkip{0.5em}
\IncMargin{0.8em}
\begin{algorithm2e}[!t]
\DontPrintSemicolon
    \textbf{Input}: datasets $\{\b{X}_t\}_{t=1}^{T}$; optional labels $\{\b{y}_t\}_{t=1}^{T}$; predicted shape $(\widehat{n}_{T+1},\widehat{d}_{T+1})$; number of candidates $M$.\;
    \textbf{Output}: candidate pool $\mathcal{C}$.\;
    \;
    \For{$t=1,\ldots,T$}{
        Map $\b{X}_t$ to the predicted target dimension:\;
        $\qquad \widetilde{\b{X}}_t
        :=
        \mathcal{A}_{t\to \widehat{d}_{T+1}}(\b{X}_t).$
    }
    Initialize $\mathcal{C}:=\emptyset$.\;
    \For{$m=1,\ldots,M$}{
        Choose either balanced family generation or bounded mixture generation.\;
        \uIf{balanced family generation is selected}{
            Sample approximately equal numbers of instances from all mapped datasets: \\
            $\qquad \b{\Gamma}^{(m)}_t
            \subseteq
            \widetilde{\b{X}}_t, \qquad
            |\b{\Gamma}^{(m)}_t|
            \approx
            \frac{\widehat{n}_{T+1}}{T}.$
        }
        \ElseIf{bounded mixture generation is selected}{
            Sample: $\b{u}\sim \operatorname{Dirichlet}(1,\ldots,1).$\;
            Define: $\b{\pi}
            =
            \pi_{\min}\b{1}
            +
            (1-T\pi_{\min})\b{u}.$\;
            Draw sample counts:\;
            $(n_1,\ldots,n_T)
            \sim
            \operatorname{Multinomial}(\widehat{n}_{T+1},\b{\pi}).$\;
            Sample:\;
            $\,\,\,\,\b{\Gamma}^{(m)}_t
            \subseteq
            \widetilde{\b{X}}_t,
            \qquad
            |\b{\Gamma}^{(m)}_t|
            \approx
            \pi_t\widehat{n}_{T+1}.$
        }
        Stack and shuffle:\;
        $\qquad \b{\Gamma}^{(m)}
        :=
        \operatorname{Shuffle}
        \left(
        \b{\Gamma}^{(m)}_1
        \cup
        \cdots
        \cup
        \b{\Gamma}^{(m)}_T
        \right).$\;
        Add small Gaussian noise:\;
        $\qquad \b{\Gamma}^{(m)}
        \leftarrow
        \b{\Gamma}^{(m)}+\b{\epsilon},
        \qquad
        \b{\epsilon}\sim\mathcal{N}(\b{0},\sigma^2\b{I}).$\;
        \uIf{labels are available}{
            Construct the candidate label vector $\b{\xi}^{(m)}$ from the sampled labels.\;
            Add $(\b{\Gamma}^{(m)},\b{\xi}^{(m)})$ to $\mathcal{C}$.\;
        }
        \Else{
            Add $\b{\Gamma}^{(m)}$ to $\mathcal{C}$.\;
        }
    }
    \KwRet $\mathcal{C}$.\;

\caption{History-based candidate generation in WRF.}\label{algorithm_WRF_candidate_generation}
\end{algorithm2e}
\DecMargin{0.8em}

\SetAlCapSkip{0.5em}
\IncMargin{0.8em}
\begin{algorithm2e}[!t]
\DontPrintSemicolon
    \textbf{Input}: candidate pool $\mathcal{C}$; targets $\b{f}_{\mathrm{rule}}$, $\b{f}_{\mathrm{fam}}$; optional targets $\b{g}_{\mathrm{rule}}$, $\b{g}_{\mathrm{fam}}$; weights $\{\lambda_k\}$.\;
    \textbf{Output}: scores $\{S(\b{\Gamma}^{(m)})\}_{m=1}^{M}$.\;
    \;
    \For{each candidate $\b{\Gamma}^{(m)}\in\mathcal{C}$}{
        Compute:\;
        $\qquad \mathcal{L}_{\mathrm{rule}}(\b{\Gamma}^{(m)})
        :=
        D\big(
        \b{\phi}(\b{\Gamma}^{(m)}),
        \b{f}_{\mathrm{rule}}
        \big).$\;
        Compute:\;
        $\qquad \mathcal{L}_{\mathrm{fam}}(\b{\Gamma}^{(m)})
        :=
        D\big(
        \b{\phi}(\b{\Gamma}^{(m)}),
        \b{f}_{\mathrm{fam}}
        \big).$\;
        Compute:\;
        $\qquad \mathcal{L}_{\mathrm{last}}(\b{\Gamma}^{(m)})
        :=
        D\big(
        \b{\phi}(\b{\Gamma}^{(m)}),
        \b{\phi}(\b{X}_T)
        \big).$\;
        Compute $\mathcal{L}_{\mathrm{shape}}(\b{\Gamma}^{(m)})$ and $\mathcal{L}_{\mathrm{col}}(\b{\Gamma}^{(m)})$.\;
        \If{labels are available}{
            Compute:\;
            $\mathcal{L}_{\mathrm{sup\text{-}rule}}(\b{\Gamma}^{(m)})
            :=
            D\big(
            \b{\psi}(\b{\Gamma}^{(m)},\b{\xi}^{(m)}),
            \b{g}_{\mathrm{rule}}
            \big).$\;
            Compute:\;
            $\mathcal{L}_{\mathrm{sup\text{-}fam}}(\b{\Gamma}^{(m)})
            :=
            D\big(
            \b{\psi}(\b{\Gamma}^{(m)},\b{\xi}^{(m)}),
            \b{g}_{\mathrm{fam}}
            \big).$\;
            Compute:\;
            $\mathcal{L}_{\mathrm{sup\text{-}last}}(\b{\Gamma}^{(m)})
            :=
            D\big(
            \b{\psi}(\b{\Gamma}^{(m)},\b{\xi}^{(m)}),
            \b{\psi}(\b{X}_T,\b{y}_T)
            \big).$\;
        }
    }
    \For{each loss type $\mathcal{L}_k$}{
        Normalize $\mathcal{L}_k$ over the candidate pool:\;
        $\widetilde{\mathcal{L}}_k(\b{\Gamma}^{(m)})
        =
        \operatorname{clip}
        \left(
        \frac{
        \mathcal{L}_k(\b{\Gamma}^{(m)})-Q_{0.05}(\mathcal{L}_k)
        }{
        Q_{0.95}(\mathcal{L}_k)-Q_{0.05}(\mathcal{L}_k)+\varepsilon
        },
        0,1
        \right).$
    }
    \For{each candidate $\b{\Gamma}^{(m)}\in\mathcal{C}$}{
        Compute the total candidate score:\;
        $\qquad S(\b{\Gamma}^{(m)})
        :=
        \frac{
        \sum_k \lambda_k
        \widetilde{\mathcal{L}}_k(\b{\Gamma}^{(m)})
        }{
        \sum_k \lambda_k+\varepsilon
        }.$
    }
    \KwRet $\{S(\b{\Gamma}^{(m)})\}_{m=1}^{M}$.\;

\caption{Candidate scoring in WRF.}\label{algorithm_WRF_candidate_scoring}
\end{algorithm2e}
\DecMargin{0.8em}

\SetAlCapSkip{0.5em}
\IncMargin{0.8em}
\begin{algorithm2e}[!t]
\DontPrintSemicolon
    \textbf{Input}: selected candidate $\b{\Gamma}^{(0)}$; optional labels $\b{\xi}^{(0)}$; refinement iterations $R$; weights $\{\beta_k\}$.\;
    \textbf{Output}: refined evolved dataset $\widehat{\b{X}}_{T+1}$.\;
    \;
    Initialize the optimization variable by the selected candidate: $\b{\Gamma}^{(0)} \leftarrow \b{\Gamma}^{(0)}.$\;
    \For{$r=0,\ldots,R-1$}{
        Compute the differentiable unsupervised descriptor: $\widetilde{\b{\phi}}(\b{\Gamma}^{(r)}).$\;
        \If{labels are available}{
            Compute the differentiable supervised descriptor: $\widetilde{\b{\psi}}(\b{\Gamma}^{(r)},\b{\xi}^{(0)}).$\;
        }
        Compute the differentiable refinement objective:\;
        $\,\,\,\,\mathcal{J}(\b{\Gamma}^{(r)})
        =
        \mathcal{J}_{\mathrm{unsup}}(\b{\Gamma}^{(r)})
        +
        \eta_{\mathrm{sup}}
        \mathcal{J}_{\mathrm{sup}}(\b{\Gamma}^{(r)},\b{\xi}^{(0)}).$\;
        Update by Adam:\;
        $\b{\Gamma}^{(r+1)}
        =
        \operatorname{AdamStep}
        \left(
        \b{\Gamma}^{(r)},
        \nabla_{\b{\Gamma}}
        \mathcal{J}(\b{\Gamma}^{(r)})
        \right).$
    }
    Select the best refinement iterate:\;
    $\qquad r^{\star}
    :=
    \arg\min_{0\leq r\leq R}
    \mathcal{J}(\b{\Gamma}^{(r)}).$\;
    Set: $\widehat{\b{X}}_{T+1}
    :=
    \b{\Gamma}^{(r^{\star})}.$\;
    \KwRet $\widehat{\b{X}}_{T+1}$.\;

\caption{Dataset refinement in WRF.}\label{algorithm_WRF_refinement}
\end{algorithm2e}
\DecMargin{0.8em}

\subsection{Time Complexity of the WRF Algorithm}

In the following, we analyze the time complexity of WRF. Let:
\begin{align}
N := \sum_{t=1}^{T} n_t,
\end{align}
denote the total number of historical samples, and let:
\begin{align}
d_{\max} := \max_{1\leq t\leq T} d_t,
\end{align}
denote the maximum historical dimension. We also denote the predicted sample size and dimensionality of the evolved dataset by:
\begin{align}
\widehat{n}:=\widehat{n}_{T+1},
\qquad
\widehat{d}:=\widehat{d}_{T+1}.
\end{align}
Let \(M\) be the number of candidate datasets, and let \(R\) be the number of refinement iterations. Since the descriptor dimensions \(p\) and \(q\) are fixed constants in the WRF algorithm, they do not dominate the asymptotic complexity.

\begin{lemma}[Time complexity of descriptor computation]
For a dataset:
\[
\b{X}_t\in\mathbb{R}^{n_t\times d_t},
\]
with sample size $n_t$ and dimensionality $d_t$, the time complexity of computing the unsupervised descriptor \(\b{\phi}(\b{X}_t)\) is:
\begin{align}
\mathcal{O}
\left(
n_t^2d_t+n_td_t^2+d_t^3
\right),
\end{align}
up to the cost of the \(k\)-means iterations used in the silhouette terms. Therefore, computing the descriptors for all historical datasets costs:
\begin{align}
\mathcal{O}
\left(
\sum_{t=1}^{T}
\left(
n_t^2d_t+n_td_t^2+d_t^3
\right)
\right).
\end{align}
\end{lemma}

\begin{proof}
The descriptor \(\b{\phi}(\b{X}_t)\) contains pairwise-distance summaries, covariance statistics, principal variance statistics, and clustering-tendency statistics. Computing all pairwise distances among \(n_t\) samples in \(d_t\) dimensions costs:
\[
\mathcal{O}(n_t^2d_t).
\]
Computing the covariance matrix of the standardized dataset costs:
\[
\mathcal{O}(n_td_t^2),
\]
and computing its eigenvalues costs:
\[
\mathcal{O}(d_t^3).
\]
The \(k\)-means and silhouette terms are computed only for fixed values \(k\in\{2,3,4,5\}\). Since this set of \(k\)'s is constant, these terms do not change the dependence on the descriptor dimension \(p\), although they add the usual clustering iteration cost. Hence, the dominant descriptor cost for one dataset is:
\[
\mathcal{O}
\left(
n_t^2d_t+n_td_t^2+d_t^3
\right).
\]
Summing this cost over all historical datasets gives:
\[
\mathcal{O}
\left(
\sum_{t=1}^{T}
\left(
n_t^2d_t+n_td_t^2+d_t^3
\right)
\right).
\]
\end{proof}

\begin{lemma}[Time complexity of rule prediction and shape prediction]
The rule-following target descriptors, family-resemblance target descriptors, and predicted matrix shape can be computed in a time complexity:
\begin{align}
\mathcal{O}(T).
\end{align}
\end{lemma}

\begin{proof}
The rule-following targets are computed by second-order extrapolation of the descriptor sequence. Since the descriptor dimensions \(p\) and \(q\) are fixed constants, each descriptor operation costs constant time with respect to \(n_t\) and \(d_t\). Computing the family-resemblance target requires a weighted average over \(T\) descriptors, which costs \(\mathcal{O}(T)\). Similarly, predicting the next shape requires processing the sequence of sample sizes and dimensions. 
If the whole sequence is processed, this step is \(\mathcal{O}(T)\). 
Therefore, the overall cost of this stage is \(\mathcal{O}(T)\).
\end{proof}

\begin{lemma}[Time complexity of dimensionality alignment]
Mapping all historical datasets to the predicted target dimension \(\widehat{d}\) costs at most:
\begin{align}
\mathcal{O}
\left(
\sum_{t=1}^{T}
(n_td_t^2+d_t^3)
\right),
\end{align}
when PCA is used for dimensionality reduction. If only copying or padding is needed, the cost is lower.
\end{lemma}

\begin{proof}
For each dataset \(\b{X}_t\), if \(d_t>\widehat{d}\), the algorithm may reduce the dimension using PCA. Computing the covariance matrix costs \(\mathcal{O}(n_td_t^2)\), and computing the eigendecomposition costs \(\mathcal{O}(d_t^3)\). Thus, the PCA-based alignment cost for \(\b{X}_t\) is:
\[
\mathcal{O}(n_td_t^2+d_t^3).
\]
Summing over all \(T\) datasets gives:
\[
\mathcal{O}
\left(
\sum_{t=1}^{T}
(n_td_t^2+d_t^3)
\right).
\]
If \(d_t=\widehat{d}\), the dataset is copied, and if \(d_t<\widehat{d}\), the dataset is padded with small random coordinates. These cases are cheaper than PCA-based reduction.
\end{proof}

\begin{lemma}[Time complexity of candidate generation]
After dimensionality alignment, generating \(M\) candidates costs:
\begin{align}
\mathcal{O}(M\widehat{n}\widehat{d}).
\end{align}
\end{lemma}

\begin{proof}
Each candidate dataset contains approximately \(\widehat{n}\) samples in \(\widehat{d}\) dimensions. Candidate generation samples rows from the aligned historical datasets, stacks the sampled rows, shuffles them, and adds small Gaussian noise. These operations are linear in the number of entries of a candidate matrix, namely:
\[
\mathcal{O}(\widehat{n}\widehat{d}),
\]
per candidate. Since there are \(M\) candidates, the total candidate generation cost is:
\[
\mathcal{O}(M\widehat{n}\widehat{d}).
\]
\end{proof}

\begin{lemma}[Time complexity of candidate scoring]
Scoring \(M\) candidate datasets costs:
\begin{align}
\mathcal{O}
\left(
M
\left(
\widehat{n}^2\widehat{d}
+
\widehat{n}\widehat{d}^2
+
\widehat{d}^3
\right)
\right).
\end{align}
\end{lemma}

\begin{proof}
Each candidate \(\b{\Gamma}^{(m)}\in\mathbb{R}^{\widehat{n}\times\widehat{d}}\) is scored by computing its descriptor and comparing it to the rule-following, family-resemblance, and last-dataset targets. The descriptor computation for one candidate requires pairwise-distance summaries, covariance statistics, and covariance-spectrum statistics. The pairwise-distance computation costs:
\[
\mathcal{O}(\widehat{n}^2\widehat{d}),
\]
the covariance computation costs:
\[
\mathcal{O}(\widehat{n}\widehat{d}^2),
\]
and the covariance eigenvalue computation costs:
\[
\mathcal{O}(\widehat{d}^3).
\]
Therefore, scoring one candidate costs:
\[
\mathcal{O}
\left(
\widehat{n}^2\widehat{d}
+
\widehat{n}\widehat{d}^2
+
\widehat{d}^3
\right).
\]
Multiplying by \(M\) gives:
\[
\mathcal{O}
\left(
M
\left(
\widehat{n}^2\widehat{d}
+
\widehat{n}\widehat{d}^2
+
\widehat{d}^3
\right)
\right).
\]
The percentile normalization of the loss values over the candidate pool costs only \(\mathcal{O}(M)\) for each fixed loss component, which is lower-order compared to descriptor computation.
\end{proof}

\begin{lemma}[Time complexity of differentiable refinement]
If the optional refinement step is used for \(R\) iterations, its time complexity is approximately:
\begin{align}
\mathcal{O}
\left(
R
\left(
\widehat{n}^2\widehat{d}
+
\widehat{n}\widehat{d}^2
+
\widehat{d}^3
\right)
\right).
\end{align}
If refinement is not used, this term is removed.
\end{lemma}

\begin{proof}
In each refinement iteration, the algorithm computes the differentiable descriptor of the current candidate and updates the coordinates by gradient-based optimization. The differentiable unsupervised descriptor contains pairwise-distance summaries, covariance trace, and covariance-spectrum information. Therefore, one refinement iteration has the same dominant costs as candidate descriptor computation:
\[
\mathcal{O}
\left(
\widehat{n}^2\widehat{d}
+
\widehat{n}\widehat{d}^2
+
\widehat{d}^3
\right).
\]
The supervised differentiable descriptor adds class-centroid, within-class, and between-class computations. These terms are lower-order than the full pairwise-distance computation when pairwise distances are explicitly used. Hence, for \(R\) refinement iterations, the total refinement cost is:
\[
\mathcal{O}
\left(
R
\left(
\widehat{n}^2\widehat{d}
+
\widehat{n}\widehat{d}^2
+
\widehat{d}^3
\right)
\right).
\]
\end{proof}

\begin{proposition}[Overall time complexity of WRF]
The overall time complexity of WRF is:
\begin{equation}
\begin{aligned}
\mathcal{O}
\Bigg(
&\sum_{t=1}^{T}
\left(
n_t^2d_t+n_td_t^2+d_t^3
\right)
\\
&+
(M+R)
\left(
\widehat{n}^2\widehat{d}
+
\widehat{n}\widehat{d}^2
+
\widehat{d}^3
\right)
\Bigg),
\end{aligned}
\end{equation}
where \(R=0\) if differentiable refinement is not used.
\end{proposition}

\begin{proof}
The total cost is obtained by adding the costs of the main stages. Historical descriptor computation costs:
\[
\mathcal{O}
\left(
\sum_{t=1}^{T}
\left(
n_t^2d_t+n_td_t^2+d_t^3
\right)
\right).
\]
Rule prediction and shape prediction cost \(\mathcal{O}(T)\), which is lower-order compared to descriptor computation. Candidate generation costs:
\[
\mathcal{O}(M\widehat{n}\widehat{d}),
\]
which is lower-order compared to candidate scoring when pairwise-distance descriptors are computed. Candidate scoring costs:
\[
\mathcal{O}
\left(
M
\left(
\widehat{n}^2\widehat{d}
+
\widehat{n}\widehat{d}^2
+
\widehat{d}^3
\right)
\right).
\]
Optional refinement costs:
\[
\mathcal{O}
\left(
R
\left(
\widehat{n}^2\widehat{d}
+
\widehat{n}\widehat{d}^2
+
\widehat{d}^3
\right)
\right).
\]
Adding these dominant terms gives
\begin{equation*}
\begin{aligned}
\mathcal{O}
\Bigg(
&\sum_{t=1}^{T}
\left(
n_t^2d_t+n_td_t^2+d_t^3
\right)
\\
&+
(M+R)
\left(
\widehat{n}^2\widehat{d}
+
\widehat{n}\widehat{d}^2
+
\widehat{d}^3
\right)
\Bigg),
\end{aligned}
\end{equation*}
\end{proof}

\subsection{Space Complexity of the WRF Algorithm}

We now analyze the space complexity of WRF. The memory usage depends on whether all candidate datasets are stored simultaneously or candidates are generated and scored sequentially.

\begin{lemma}[Space complexity of historical data storage]
The space required to store the historical datasets is:
\begin{align}
\mathcal{O}
\left(
\sum_{t=1}^{T} n_td_t
\right).
\end{align}
If all dimension-aligned historical datasets are also stored, the additional space is:
\begin{align}
\mathcal{O}(N\widehat{d}).
\end{align}
\end{lemma}

\begin{proof}
The dataset \(\b{X}_t\in\mathbb{R}^{n_t\times d_t}\) contains \(n_td_t\) scalar entries. Therefore, storing all historical datasets requires:
\[
\sum_{t=1}^{T} n_td_t,
\]
scalar entries, which gives:
\[
\mathcal{O}
\left(
\sum_{t=1}^{T} n_td_t
\right).
\]
After mapping each historical dataset to the predicted target dimension \(\widehat{d}\), the aligned datasets contain:
\[
\sum_{t=1}^{T} n_t\widehat{d}
=
N\widehat{d},
\]
entries. Thus, storing all aligned datasets requires:
\[
\mathcal{O}(N\widehat{d}),
\]
additional memory.
\end{proof}

\begin{lemma}[Space complexity of the candidate pool]
If all \(M\) candidates are stored simultaneously, the candidate pool requires:
\begin{align}
\mathcal{O}(M\widehat{n}\widehat{d}),
\end{align}
memory in the unsupervised case. In the supervised case, storing candidate labels requires an additional:
\begin{align}
\mathcal{O}(M\widehat{n}),
\end{align}
memory.
\end{lemma}

\begin{proof}
Each candidate dataset has approximately \(\widehat{n}\) samples and \(\widehat{d}\) dimensions. Therefore, one candidate contains:
\[
\widehat{n}\widehat{d},
\]
scalar entries. Storing \(M\) candidates requires:
\[
M\widehat{n}\widehat{d},
\]
scalar entries, giving:
\[
\mathcal{O}(M\widehat{n}\widehat{d}).
\]
In the supervised case, each candidate also has a label vector of length \(\widehat{n}\). Storing labels for all \(M\) candidates therefore requires:
\[
\mathcal{O}(M\widehat{n}),
\]
additional memory.
\end{proof}

\begin{lemma}[Space complexity of descriptor computation]
The temporary space required to compute descriptors for one candidate is:
\begin{align}
\mathcal{O}(\widehat{n}^2+\widehat{d}^2).
\end{align}
The space required to store all historical and candidate descriptors is:
\begin{align}
\mathcal{O}\big(T(p+q)+M(p+q)\big).
\end{align}
\end{lemma}

\begin{proof}
Computing pairwise-distance summaries may require storing the pairwise-distance matrix or the list of pairwise distances. For \(\widehat{n}\) samples, this requires:
\[
\mathcal{O}(\widehat{n}^2),
\]
temporary memory. Computing covariance statistics requires storing a covariance matrix of size \(\widehat{d}\times \widehat{d}\), which requires:
\[
\mathcal{O}(\widehat{d}^2),
\]
memory. Hence, the temporary descriptor-computation memory for one candidate is:
\[
\mathcal{O}(\widehat{n}^2+\widehat{d}^2).
\]
The historical descriptors have dimension \(p\) in the unsupervised case and \(q\) in the supervised case. Since there are \(T\) historical datasets and \(M\) candidates, storing all descriptors requires:
\[
\mathcal{O}\big(T(p+q)+M(p+q)\big).
\]
Because \(p\) and \(q\) are fixed constants in the current implementation, this descriptor-storage term is usually lower-order.
\end{proof}

\begin{lemma}[Space complexity of differentiable refinement]
During differentiable refinement, the memory required for the optimization variable, its gradients, and temporary descriptor computations is approximately:
\begin{align}
\mathcal{O}
\left(
\widehat{n}\widehat{d}
+
\widehat{n}^2
+
\widehat{d}^2
\right).
\end{align}
\end{lemma}

\begin{proof}
The optimization variable is the candidate matrix $\b{\Gamma}\in\mathbb{R}^{\widehat{n}\times\widehat{d}}$ which requires:
\[
\mathcal{O}(\widehat{n}\widehat{d}),
\]
memory. Its gradient has the same size and therefore the same order of memory. The differentiable descriptor computation may store pairwise distances among the \(\widehat{n}\) generated samples, requiring:
\[
\mathcal{O}(\widehat{n}^2),
\]
temporary memory. It may also store the covariance matrix, requiring:
\[
\mathcal{O}(\widehat{d}^2),
\]
memory. Therefore, the total refinement memory is approximately:
\[
\mathcal{O}
\left(
\widehat{n}\widehat{d}
+
\widehat{n}^2
+
\widehat{d}^2
\right).
\]
\end{proof}

\begin{proposition}[Overall space complexity of WRF]
If all candidate datasets are stored simultaneously, the total space complexity of WRF is approximately:
\begin{align}
\mathcal{O}
\left(
\sum_{t=1}^{T} n_td_t
+
N\widehat{d}
+
M\widehat{n}\widehat{d}
+
M\widehat{n}
+
\widehat{n}^2
+
\widehat{d}^2
\right).
\end{align}
In the unsupervised case, the label-storage term \(\mathcal{O}(M\widehat{n})\) is removed.
\end{proposition}

\begin{proof}
The historical datasets require:
\[
\mathcal{O}
\left(
\sum_{t=1}^{T} n_td_t
\right),
\]
memory. The aligned historical datasets require:
\[
\mathcal{O}(N\widehat{d}),
\]
memory if stored. The candidate pool requires:
\[
\mathcal{O}(M\widehat{n}\widehat{d}),
\]
memory. In the supervised case, storing all candidate labels requires:
\[
\mathcal{O}(M\widehat{n}),
\]
additional memory. Descriptor computation and refinement may require temporary memory:
\[
\mathcal{O}(\widehat{n}^2+\widehat{d}^2).
\]
Adding these terms gives:
\[
\mathcal{O}
\left(
\sum_{t=1}^{T} n_td_t
+
N\widehat{d}
+
M\widehat{n}\widehat{d}
+
M\widehat{n}
+
\widehat{n}^2
+
\widehat{d}^2
\right).
\]
If labels are not used, the term \(\mathcal{O}(M\widehat{n})\) is removed.
\end{proof}

\begin{remark}[Sequential candidate generation]
If candidates are generated and scored sequentially rather than stored simultaneously, the candidate-pool memory can be reduced from:
\[
\mathcal{O}(M\widehat{n}\widehat{d}),
\]
to:
\[
\mathcal{O}(\widehat{n}\widehat{d}),
\]
apart from the memory needed to store candidate scores and the best candidate found so far.
\end{remark}

\section{Simulations}\label{section_simulations}

We evaluated the proposed WRF data evolution method using both synthetic datasets \cite{pedregosa2011scikit} and the MNIST digit dataset \cite{lecun1998gradient}.
The Python code of simulations is available in \url{https://github.com/aghojogh/WRF_data_evolution}. 

\subsection{Simulations on Synthetic Datasets}

The synthetic datasets used in the simulations are common benchmark structures for evaluating clustering, nonlinear representation, and generative behavior in machine learning \cite{pedregosa2011scikit}.

As shown in Fig. \ref{figure_toy_experiment1}, we simulated the WRF algorithm on a sequence of synthetic datasets:
\begin{align*}
\b{X}_1 = \text{Moons},\, \b{X}_2 = \text{Blobs},\, \b{X}_3 = \text{Circles}, ...
\end{align*}
The goal is to evolve data to obtain $\b{X}_4$. 
The results of data evolution by the WRF algorithm are illustrated in Fig. \ref{figure_toy_experiment1} for both supervised and unsupervised settings, where the two-dimensional plots of datasets are projections of the datasets onto their first and second PCA components \cite{ghojogh2023principal}. In the supervised WRF, the labels are used in the algorithm and the classes are colored by different colors in the plots. 
As shown in the figure, data refinement has modified the evolved dataset slightly using gradient-descent Adam optimization. 

The sample size and dimensionality of the datasets are as follows in this experiment:
\begin{itemize}\itemsep0em 
\item $\b{X}_1$: $n_1 = 800, d_1 = 2$
\item $\b{X}_2$: $n_2 = 1100, d_2 = 3$
\item $\b{X}_3$: $n_3 = 1400, d_3 = 4$
\end{itemize}
The sample size and dimensionality of the evolved data $\b{X}_4$ are $n_4 = 1700$ and $d_4 = 5$, which is natural given the increasing trend in the sequence.

As shown on the right-hand side of Fig. \ref{figure_toy_experiment1}, the evolved dataset $\b{X}_4$ contains the patterns from the datasets $\b{X}_1$, $\b{X}_2$, and $\b{X}_3$. 
The evolved dataset has the patterns of moons (shown in red curves), blobs (shown in green curves), and circles (shown in blue curve).

We did another simulation with another sequence of datasets, as depicted in Fig. \ref{figure_toy_experiment2}. The datasets in this sequence are:
\begin{align*}
\b{X}_1 = \text{S curve},\, \b{X}_2 = \text{Moons},\, \b{X}_3 = \text{Blobs}, ...
\end{align*}
The sample size and dimensionality of the datasets are similar to the previous simulation. Again, the sample size and dimensionality of the evolved data $\b{X}_4$ are $n_4 = 1700$ and $d_4 = 5$, which are expected.

The evolved dataset $\b{X}_4$ is illustrated in Fig. \ref{figure_toy_experiment2}. 
As shown in the figure, data refinement has modified the evolved dataset by deforming the geometry of the evolution slightly. 
As shown on the right-hand side of this figure, the evolved dataset $\b{X}_4$ contains the patterns from the datasets $\b{X}_1$, $\b{X}_2$, and $\b{X}_3$, i.e., it has the patterns of S curve (shown in orange and purple curves), moons (shown in red curves), and blobs (shown in green curves).

\begin{figure*}[!t]
\centering
\includegraphics[width=6.7in]{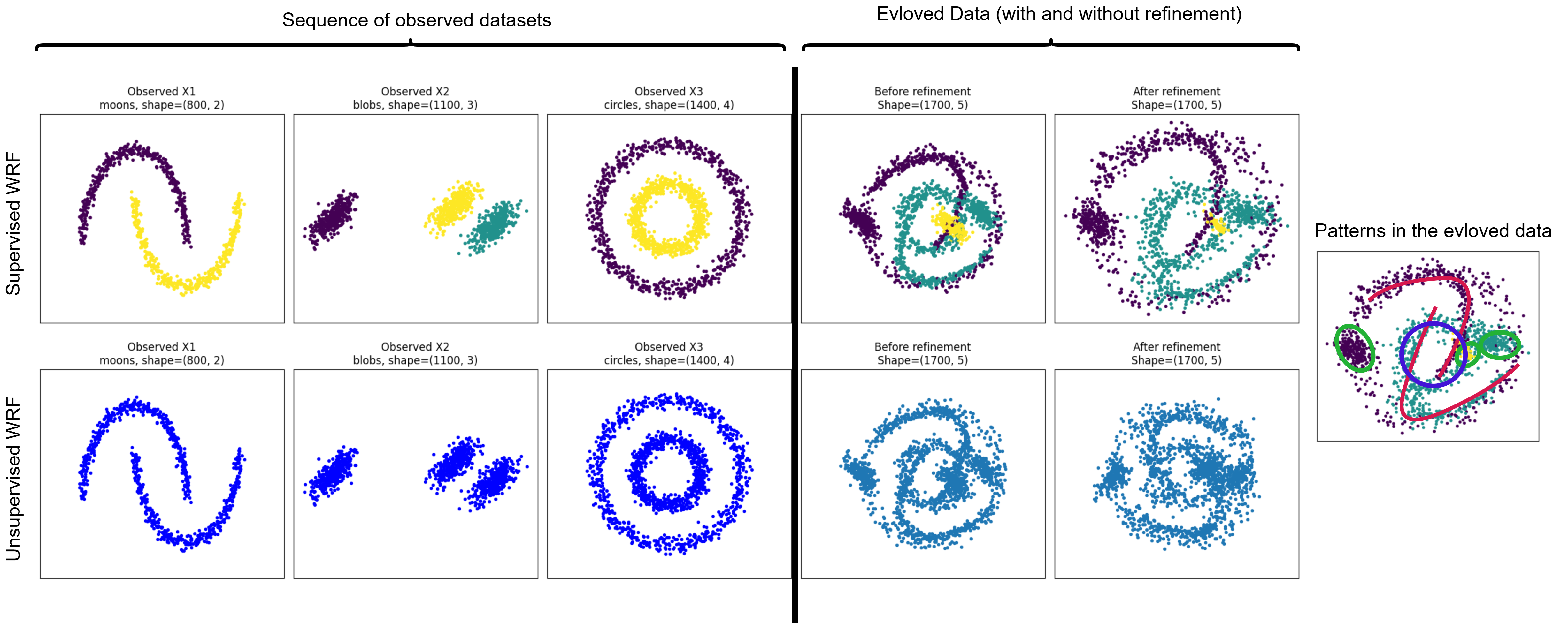}
\caption{Simulation of the WRF data evolution algorithm on a sequence of synthetic datasets, (1) moons, (2) blobs, and (3) circles, in both unsupervised and supervised settings. 
The two-dimensional plots of datasets are projections of the datasets onto their first and second PCA components.
As shown on the right-hand side, the evolved dataset $\b{X}_4$ contains the patterns from the datasets $\b{X}_1$, $\b{X}_2$, and $\b{X}_3$, i.e., it has the patterns of moons (shown in red curves), blobs (shown in green curves), and circles (shown in blue curve).}
\label{figure_toy_experiment1}
\end{figure*}

\begin{figure*}[!t]
\centering
\includegraphics[width=6.7in]{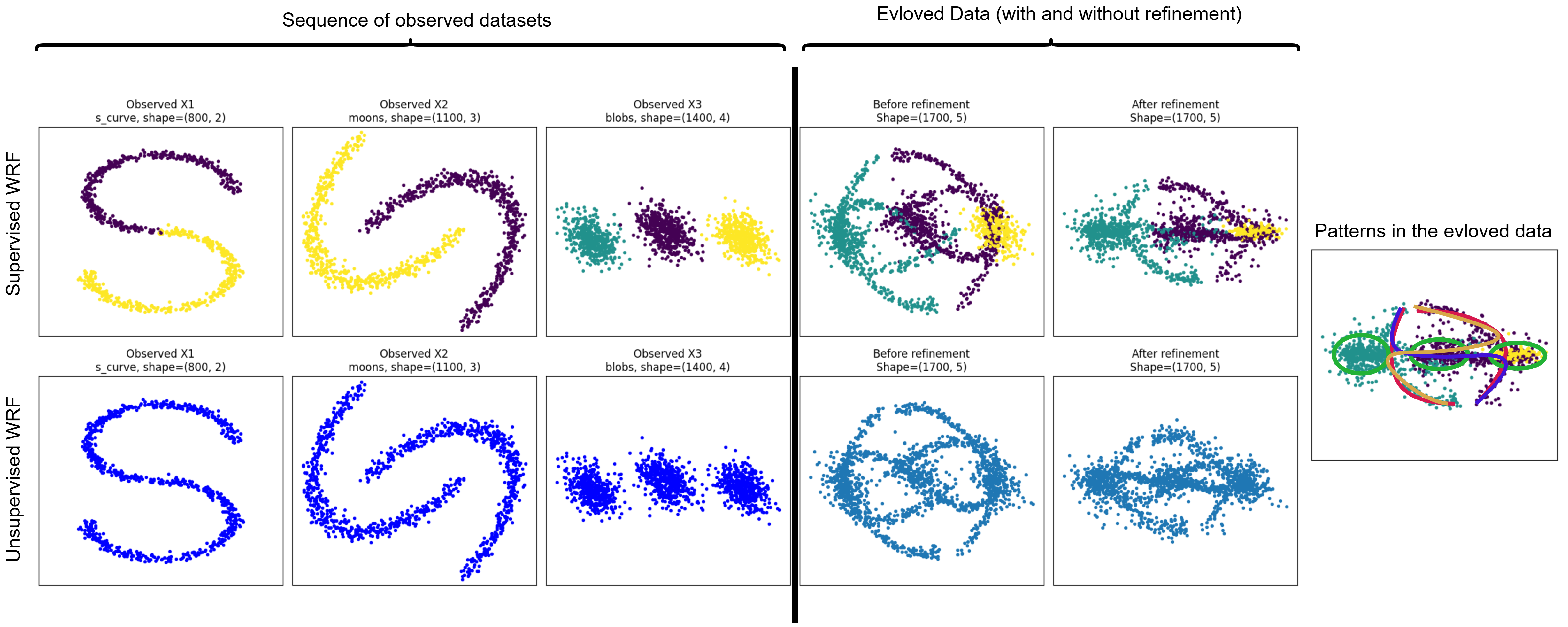}
\caption{Simulation of the WRF data evolution algorithm on a sequence of synthetic datasets, (1) S curve, (2) moons, and (3) blobs, in both unsupervised and supervised settings. 
The two-dimensional plots of datasets are projections of the datasets onto their first and second PCA components.
As shown on the right-hand side, the evolved dataset $\b{X}_4$ contains the patterns from the datasets $\b{X}_1$, $\b{X}_2$, and $\b{X}_3$, i.e., it has the patterns of S curve (shown in orange and purple curves), moons (shown in red curves), and blobs (shown in green curves).}
\label{figure_toy_experiment2}
\end{figure*}

\subsection{Simulations on MNIST Digit Dataset}

We also evaluated the WRF data evolution algorithm using the MNIST digit dataset \cite{lecun1998gradient}. 
We created a sequence of digit datasets as follows:
\begin{itemize}\itemsep0em 
\item $\b{X}_1$: digits $\{0,1\}$, sample size: $600$, dimensionality: $28 \times 28 = 784$ 
\item $\b{X}_2$: digits $\{4,9\}$, sample size: $600$, dimensionality: $28 \times 28 = 784$ 
\item $\b{X}_3$: digits $\{6,7,8\}$, sample size: $900$, dimensionality: $28 \times 28 = 784$ 
\end{itemize}
For this experiment, we do not let WRF decide about the dimensionality because the images in the evolved dataset must be of dimensionality $28 \times 28 = 784$.

We performed WRF data evolution for this sequence of datasets in both unsupervised and supervised settings. 
The evolved datasets are shown in Fig. \ref{figure_mnist_experiment_unsupervised} and Fig. \ref{figure_mnist_experiment_supervised}, respectively, for unsupervised and supervised settings.
As shown in these figures, data refinement has slightly modified the evolved dataset by gradient-based optimization. 

The sample size of the evolved data, in both unsupervised and supervised settings, is $1350$, which is reasonable because the sequence exhibits an increasing pattern in sample sizes. Therefore, the sample size of the evolved data is expected to be greater than $900$.

The evolved data in the unsupervised WRF contains digits $\{1,4,6,7,9\}$, while the evolved data in the supervised WRF contains digits $\{0,1,4,6,7,8\}$. The supervised evolved dataset contains more digit classes because label information is used during the evolution process. 

\begin{figure*}[!t]
\centering
\includegraphics[width=6.7in]{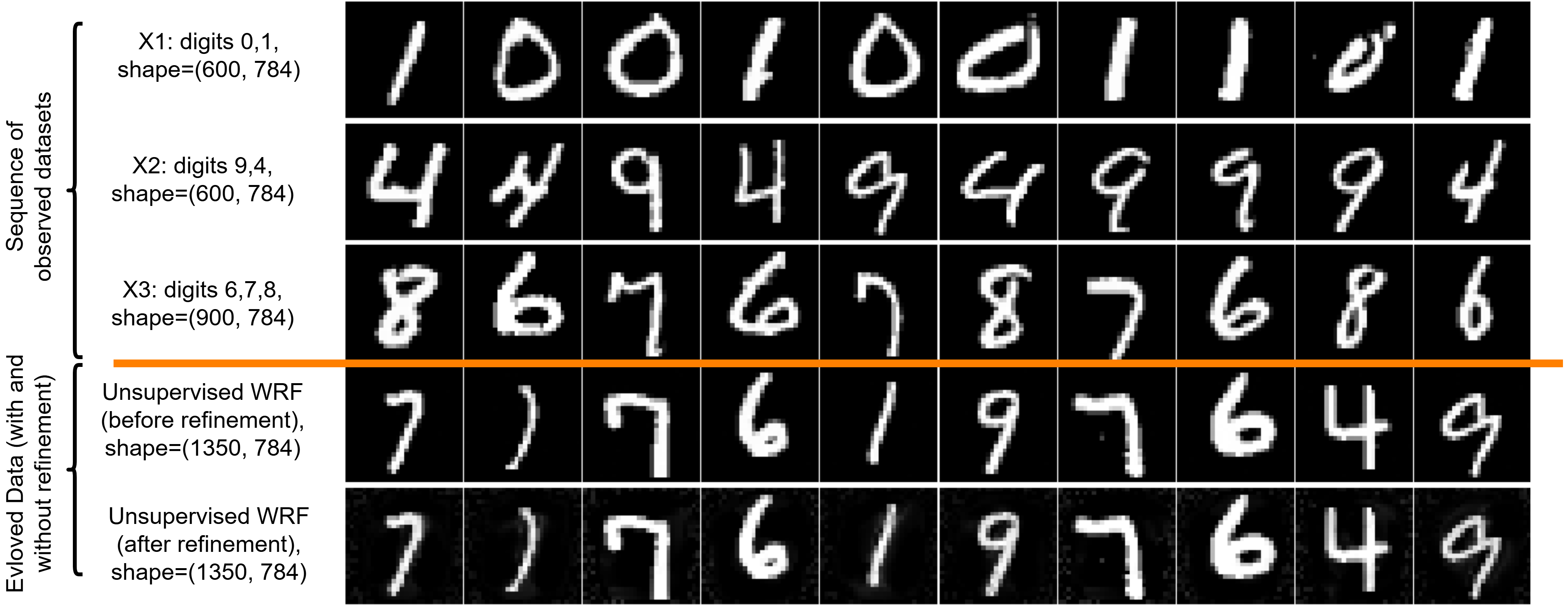}
\caption{Simulation of the WRF data evolution algorithm on a sequence of digit datasets, (1) a dataset of digits $\{0,1\}$, (2) a dataset of digits $\{4,9\}$, and (3) a dataset of digits $\{6,7,8\}$, in an \textbf{unsupervised} setting. As shown on the right-hand side, the evolved dataset $\b{X}_4$ contains the patterns from the datasets $\b{X}_1$, $\b{X}_2$, and $\b{X}_3$, i.e., it has the patterns of digits $\{1,4,6,7,9\}$.}
\label{figure_mnist_experiment_unsupervised}
\end{figure*}

\begin{figure*}[!t]
\centering
\includegraphics[width=6.7in]{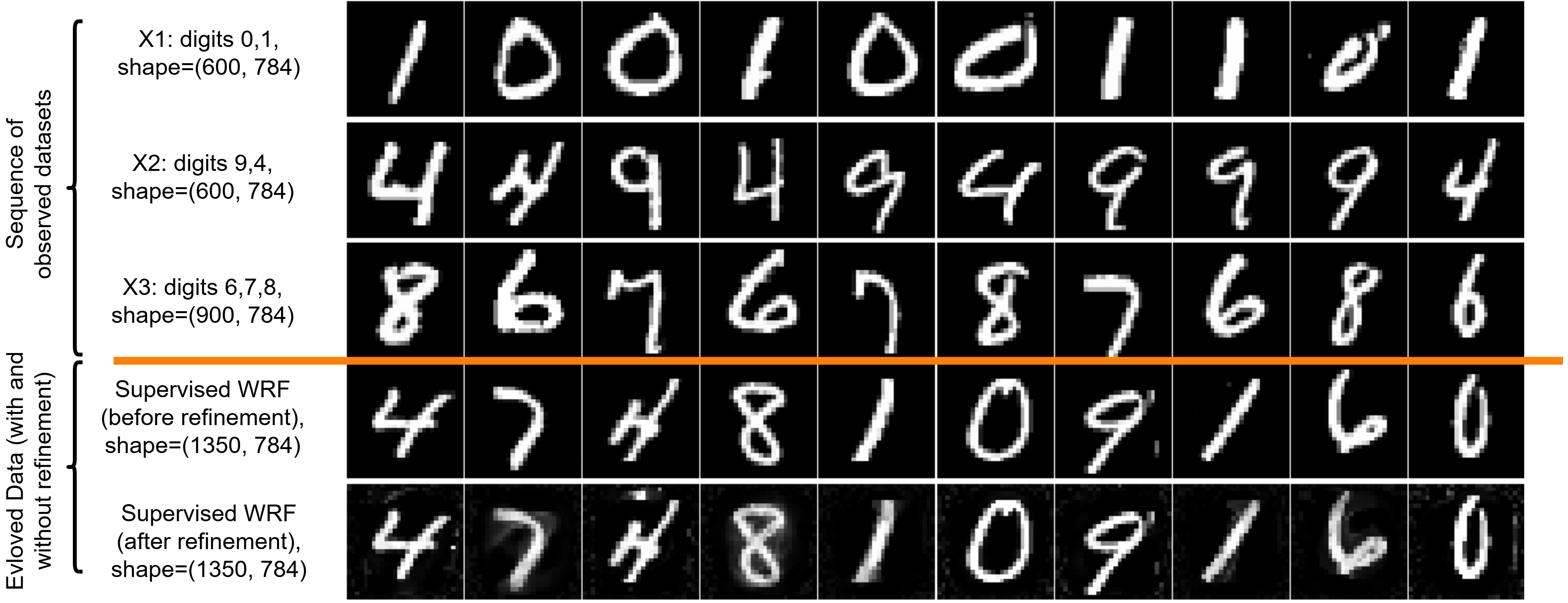}
\caption{Simulation of the WRF data evolution algorithm on a sequence of digit datasets, (1) a dataset of digits $\{0,1\}$, (2) a dataset of digits $\{4,9\}$, and (3) a dataset of digits $\{6,7,8\}$, in a \textbf{supervised} setting. As shown on the right-hand side, the evolved dataset $\b{X}_4$ contains the patterns from the datasets $\b{X}_1$, $\b{X}_2$, and $\b{X}_3$, i.e., it has the patterns of digits $\{0,1,4,6,7,8\}$.}
\label{figure_mnist_experiment_supervised}
\end{figure*}

\section{Conclusion and Future Work}\label{section_conclusion}

This paper introduced Wittgenstein's Rule Following (WRF) data evolution, a philomatically motivated framework for generating an evolved dataset from a sequence of previously observed datasets. Unlike ordinary data generation, WRF aims to continue the implicit rule expressed by a historical sequence while preserving resemblance to the family of previous datasets.

The proposed method is inspired by Wittgenstein's ideas of rule following and family resemblance. Each dataset is represented by structural descriptors rather than pointwise correspondences. These descriptors summarize geometric, distributional, clustering, and, in the supervised case, label-based properties of the data. WRF predicts rule-following and family-resemblance targets in descriptor space, generates candidate datasets from the observed history, scores them using several structural losses, and optionally refines the best candidate through differentiable optimization.

The simulations showed that WRF can generate meaningful continuations of evolving datasets in both unsupervised and supervised settings. In the synthetic experiments, the evolved datasets followed the structural trends of the historical data. In the MNIST experiment, the generated datasets preserved digit-like patterns, and the supervised version benefited from label information.

Future work can extend WRF in several directions. First, richer descriptors may be added, such as topological, graph-based, kernel-based, or neural descriptors. Second, more advanced rule-prediction models, such as Gaussian processes, recurrent neural networks, transformers, or state-space models, may replace the current second-order extrapolation. Third, candidate generation may be combined with modern generative models such as variational autoencoders, diffusion models, normalizing flows, or generative adversarial networks. Fourth, the supervised setting can be expanded to handle changing class structures, including new classes, disappearing classes, class splitting, and class merging.
Fifth, future work may introduce a discount factor for weighting historical datasets. The earlier datasets in the sequence, which are farther from the evolved dataset in the sequence, can have less contribution than the later datasets in the sequence. 

Finally, future work may study the theoretical properties of WRF, including stability, sensitivity to descriptor choices, and connections to concept drift, domain adaptation, continual learning, and synthetic data generation. Overall, WRF provides an initial framework for rule-following data evolution and shows how Wittgenstein's philosophical ideas can inspire a concrete machine learning algorithm.


\bibliography{References}
\bibliographystyle{icml2016}

\onecolumn
{\small
\tableofcontents
}
\twocolumn

\end{document}